%% file: main_cvpr.tex
\documentclass[10pt,twocolumn,letterpaper]{article}

\usepackage{cvpr}              
\input{preamble}
\definecolor{cvprblue}{rgb}{0.21,0.49,0.74}
\usepackage[pagebackref,breaklinks,colorlinks,allcolors=cvprblue]{hyperref}

\usepackage{comment}
\usepackage{amsthm}
\newtheorem{definition}{Definition}
\newtheorem{theorem}{Theorem}
\newtheorem{proposition}{Proposition}
\newtheorem{lemma}{Lemma}

\usepackage{booktabs, multirow} 
\usepackage{soul}
\usepackage{xcolor,colortbl} 
\usepackage{changepage,threeparttable} 
\usepackage{docmute} 
\usepackage{float} 

\input{alldefin}

\newcommand{\vect}[1]{\mathbf{#1}}
\newcommand{\set}[1]{\mathcal{#1}}
\newcommand{\subspace}[1]{V_{#1}}
\newcommand{\basis}[1]{\set{E}_{#1}}
\newcommand{\proj}[1]{\mathbb{P}_{#1}}
\newcommand{\norm}[1]{\left\lVert #1 \right\rVert}
\newcommand{\ip}[2]{\left\langle #1, #2 \right\rangle}
\newcommand{\tree}{\set{T}}
\newcommand{\nodes}{\set{V}}

\newcommand{\distT}[2]{D_{\tree}(#1, #2)}

\newcommand{\parent}[1]{\text{parent}(#1)}

\def\paperID{14142} 
\def\confName{CVPR}
\def\confYear{2026}

\title{Hier-COS: Making Deep Features Hierarchy-aware via\\Composition of Orthogonal Subspaces}

\author{Depanshu Sani \qquad \qquad \qquad Saket Anand\\
Indraprastha Institute of Information Technology, Delhi, India \\
{\tt\small \url{https://sites.google.com/iiitd.ac.in/hier-cos}}
}

\begin{document}
\maketitle
\input{sec/0_abstract}
\input{sec/introduction}

\input{sec/lit_review}

\input{sec/methodology}

\input{sec/hops_method}

\input{sec/results}

\input{sec/discussion}
    
{
    \small
    \bibliographystyle{ieeenat_fullname}
    \bibliography{main}
}


\clearpage
\setcounter{page}{1}
\maketitlesupplementary
\input{sec/X_supp_cvpr}

\end{document}

%% file: sec/0_abstract.tex
\begin{abstract}

Traditional classifiers treat all class labels as mutually independent, thereby considering all negative classes to be equally incorrect. This approach fails severely in many real-world scenarios, where a known semantic hierarchy defines a partial order of preferences over negative classes. While hierarchy-aware feature representations have shown promise in mitigating this problem, their performance is typically assessed using metrics like Mistake Severity (MS) and Average Hierarchical Distance (AHD). In this paper, we highlight important shortcomings in existing hierarchical evaluation metrics, demonstrating that they are often incapable of measuring true hierarchical performance. Our analysis reveals that existing methods learn sub-optimal hierarchical representations, despite competitive MS and AHD scores. To counter these issues, we introduce Hierarchical Composition of Orthogonal Subspaces (Hier-COS), a novel framework for unified `hierarchy-aware fine-grained' and `hierarchical multi-level' classification. We show that Hier-COS is theoretically guaranteed to be consistent with the given hierarchy tree. Furthermore, our framework implicitly adapts the learning capacity for different classes based on their position within the hierarchy tree --- a vital property absent in existing methods. Finally, to address the limitations of evaluation metrics, we propose Hierarchically Ordered Preference Score (HOPS), a ranking-based metric that demonstrably overcomes the deficiencies of current evaluation standards. We benchmark Hier-COS on four challenging datasets, including the deep and imbalanced tieredImageNet-H (12-level) and iNaturalist-19 (7-level). Through extensive experiments, we demonstrate that Hier-COS achieves state-of-the-art performance across all hierarchical metrics for every dataset, while simultaneously beating the top-1 accuracy in all but one case. Lastly, we show that Hier-COS can effectively learn to transform the frozen features extracted from a pretrained backbone (ViT) to be hierarchy-aware, yielding substantial benefits for hierarchical classification performance.

\end{abstract}

%% file: sec/introduction.tex
\section{Introduction}\label{sec:intro}

\begin{figure*}[t!]
    \centering
    \begin{subfigure}[b]{0.245\linewidth}
        \centering
        \includegraphics[width=\linewidth]{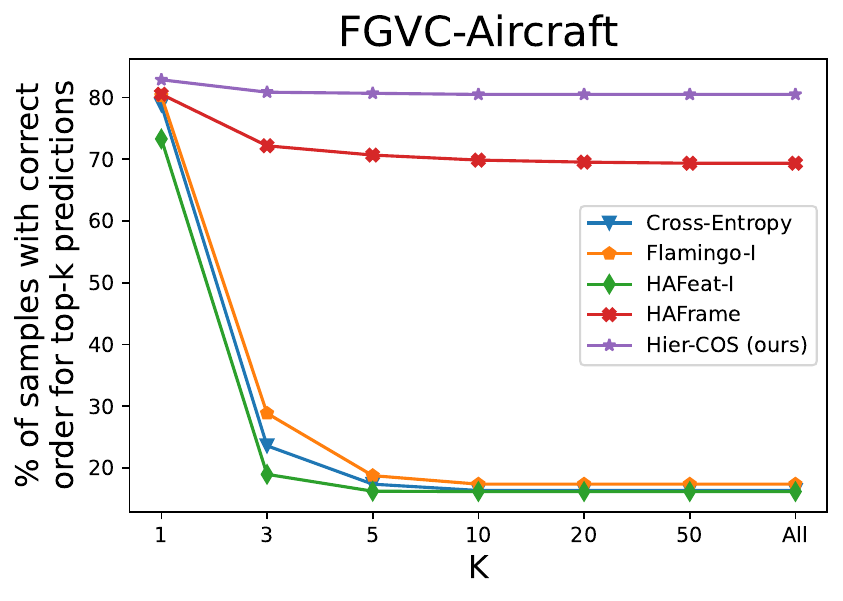}
        \caption{}
    \end{subfigure}
    \begin{subfigure}[b]{0.245\linewidth}
        \centering
        \includegraphics[width=\linewidth]{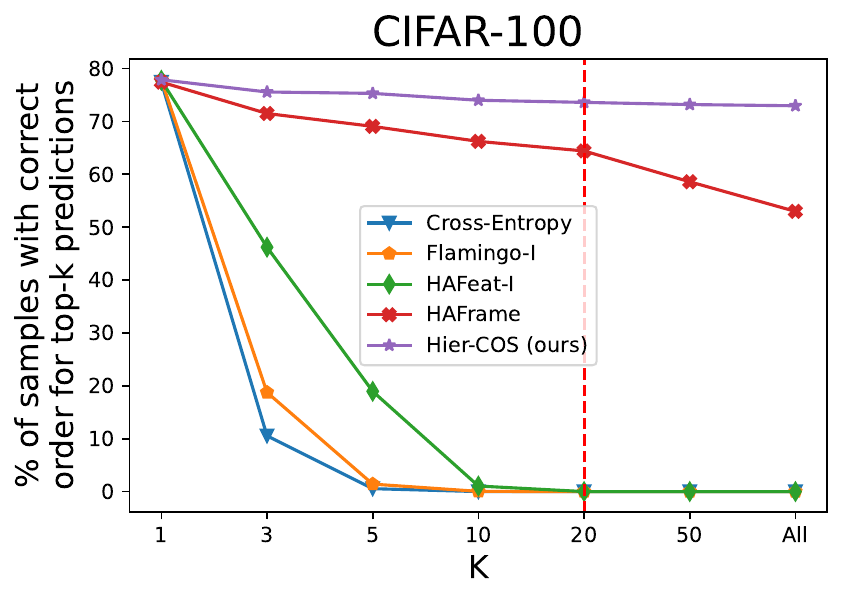}
        \caption{}
    \end{subfigure}
    \begin{subfigure}[b]{0.245\linewidth}
        \centering
        \includegraphics[width=\linewidth]{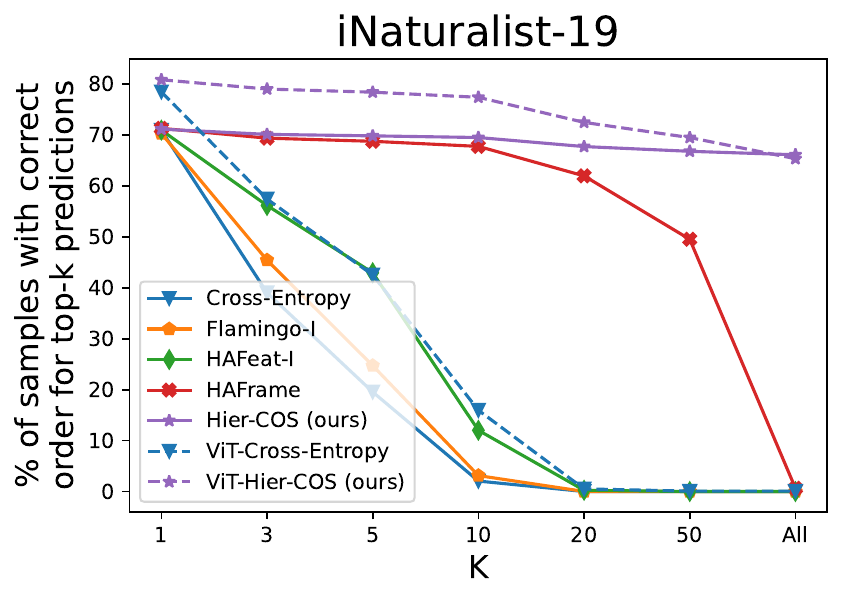}
        \caption{}
    \end{subfigure}
    \begin{subfigure}[b]{0.245\linewidth}
        \centering
        \includegraphics[width=\linewidth]{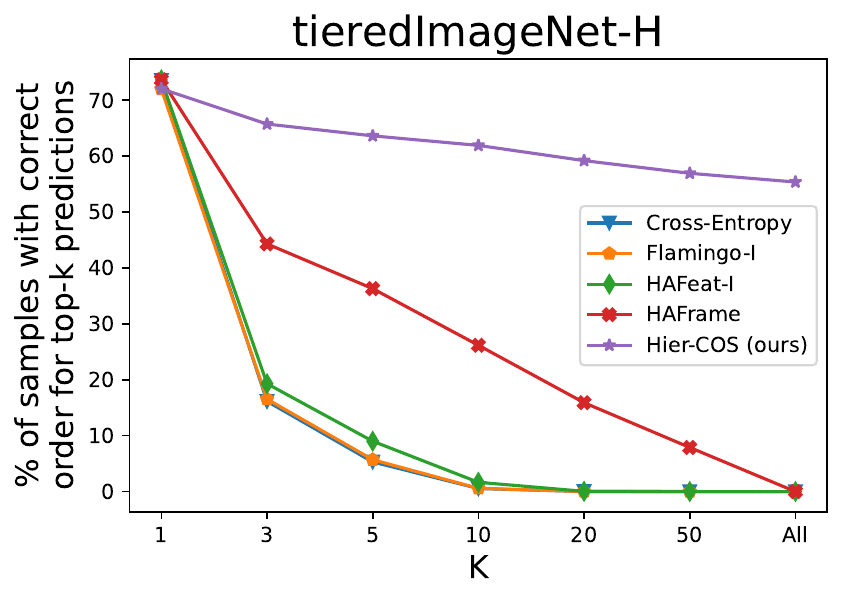}
        \caption{}
    \end{subfigure}
    \begin{subfigure}[b]{\linewidth}
        \centering
        \includegraphics[width=0.96\linewidth]{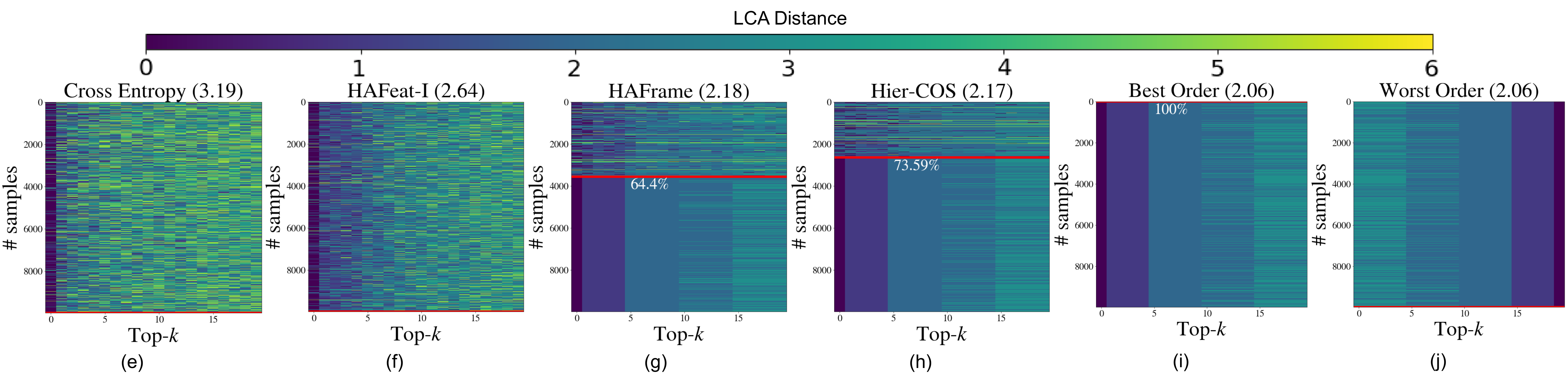}
    \end{subfigure}
    \caption{(Best viewed in color) Hierarchy-aware feature representations are evaluated based on metrics like MS and AHD@$k$ paired with top-1 or top-$k$ accuracy. Metrics like AHD@$k$ are a \textit{permutation invariant} statistic and therefore inadequate for evaluating the top-$k$ predictions, which are rank-ordered. By comparing the top-$k$ prediction order with their \emph{LCA-based} partial preference order, we point out that existing methods inadequately capture the label hierarchy structure. Across four datasets, subfigures (a)–-(d) show that the percentage of samples with a correct top-$k$ prediction order decreases significantly as $k$ increases, especially for complex datasets with deep hierarchies. 
    In subfigures (e)--(j), we use the CIFAR-100 test samples to compare the performance of four methods, including Hier-COS, by visualizing the absolute LCA distance. Each row of the image corresponds to a test sample and the columns record the LCA distance between the top-$k$ prediction and the true class. Given the class hierarchy, subfigure (i) shows the best achievable top-$20$ prediction order, with zeros in the first column indicating a 100\% top-1 accuracy. Analogously, (j) is generated by reversing the ordering for each row to generate the `worst' top-$20$ prediction order with 0\% top-$1$ accuracy. Because of permutation invariance operation within AHD@$20$, the metric obtained for the best (i) and worst (j) orderings are identical at 2.06. This extreme case notwithstanding, we make similar observations in (e)--(h), where (e) and (f) have nearly zero samples with correct prediction order and yet AHD@$20$ decreases substantially, and only moderately for (g) despite a substantially large percentage ($\sim 64.4\%$) of samples with correct top-$20$ predictions. With (h) displaying an even higher percentage ($\sim 73.5\%$), the AHD@$20$ scores drop only by a marginal value of 0.01. This analysis points out the inability of AHD@$k$ to evaluate hierarchy-aware representations. Further exacerbating the interpretability of AHD@$k$ is its absolute value which depends on the properties of the hierarchy tree (skewness, branching factor and height), as will be seen for tieredImageNet-H.    
    }
    \label{fig:teaser}
\end{figure*}

Deep neural networks have demonstrated remarkable scalability in successfully handling thousands or even millions of classes for image classification \cite{Ridnik_ImageNet21K_pretrain,An_ICCVw21_MillionClasses}.
However, most classification algorithms operate under the assumption that all classes are mutually unrelated.
In reality, classes often exhibit structured semantic relationships \cite{krizhevsky_2009_CIFAR100, ren_tiered_18_ICLR, bai2020products10klargescaleproductrecognition}, such as taxonomies defining \texttt{is-a}, \texttt{part-of} or \texttt{type-of} relationships \cite{Ridnik_ImageNet21K_pretrain,van_inat_2018_CVPR}. 
This semantic knowledge is often available and can be used to define a partial order for the target, representing a preference order over the negative classes. This partial preference order depicts the semantic similarity between the target and the negative classes, thereby quantifying the severity of mistaking the target as a negative class. However, this valuable structure, or the partial order, remains underutilized during training. Consequently, despite impressive performance, the models trained using such algorithms can make mistakes that are highly severe.

Learning hierarchy-aware feature representations has been applied to mitigate the severity of prediction errors \cite{frome_2013_NIPS, wu_2016_ACM, barz_2019_WACV, flamingo_2021_Chang, baseline-hafeat}. Such representations ensure that classes with greater semantic similarity, as defined by a given taxonomy, are closer in the feature space, thus enabling more meaningful predictions even in the event of error. 
Such a feature space also enables hierarchical multi-level classification, i.e., predicting class labels across all the hierarchical levels \cite{flamingo_2021_Chang, baseline-hafeat}. We argue that strong hierarchy-aware feature representations should: (i) maintain the classification performance, (ii) reduce the severity of mistakes, and (iii) be hierarchically consistent, i.e., the predicted classes at all coarser levels should be ancestors of the predicted leaf class.

To assess the performance of hierarchy-aware feature representation, a multi-metric evaluation system has been widely adopted in the literature. The multi-class classification performance is typically measured using the standard top-1 accuracy metric. To evaluate the consistency in hierarchical multi-level classification, metrics such as Tree-based InConsistency Error rate (TICE) \cite{TICE} and Full Path Accuracy (FPA) \cite{park2024learninghierarchicalsemanticclassification} have been employed. TICE tests whether the prediction path exists in the tree even when the fine-grained prediction is incorrect. On the other hand, FPA measures the proportion of samples correctly classified at \emph{all} hierarchical levels. The severity of mistakes is quantified by comparing the partial preference order, obtained from the hierarchy tree, to the predicted classes, which are usually ranked in order of their prediction probability. Most prominent metrics often quantify the severity using the height of the tree rooted at the Lowest Common Ancestor (LCA) of the predicted and true classes. However, as pointed out by \cite{baseline-hafeat}, recent LCA-based metrics like Mistake Severity (MS) \cite{bertinetto_2020_CVPR}, Average Hierarchical Distance (AHD) and its top-k variant (AHD@$k$) \cite{karthik_2021_ICLR} are hard to interpret and compare across different hierarchy trees. Moreover, these metrics alone are insufficient and need to be paired with the fine-grained top-1 accuracy. The resulting multi-metric evaluation can also conflict, making it difficult to make a single decision about which model performs best. Fig.~\ref{fig:teaser} compares different hierarchical classifiers and highlights crucial limitations in current evaluation metrics. This motivates the need for a new hierarchical evaluation metric.

Existing hierarchy-aware feature learning methods generally follow two paths, (i) \textit{directly} encouraging predictions to align with the partial preference order using specially designed label embeddings \cite{ barz_2019_WACV, liu_2020_CVPR}, loss functions \cite{deng_2014_HEX_ECCV, 10.1007/978-3-031-70359-1_15, bertinetto_2020_CVPR} or architectures \cite{wu_2016_ACM, baseline-hafeat} for hierarchy-aware multi-class classification, and (ii) training separate auxiliary classifiers for each hierarchical level to \textit{indirectly} learn hierarchy-aware features by performing hierarchical multi-level classification \cite{flamingo_2021_Chang, baseline-hafeat}. The indirect approach, however, can lead to hierarchical inconsistency unless additional constraints are imposed, which in turn complicates the loss function. On the other hand, direct approaches cannot perform hierarchical multi-level classification directly.
Regardless of the direct or indirect approach, the general objective of such algorithms is to bring the feature vectors closer/farther apart if they correspond to classes with higher or lower semantic similarity. A major bottleneck shared by all these methods is that the feature representation for every class is usually confined to a one-dimensional space—along the direction of its corresponding weight vector. Consequently, for a large number of fine-grained classes, the angular separation between the semantically similar classes might become very small. Additionally, the hierarchies in the real-world can be highly imbalanced \cite{ren_tiered_18_ICLR} and might also be heavily skewed in terms of its depth and branching factor. Therefore, some fine-grained classes might be more complex to separate than others. For instance, in Fig.~\ref{fig:HOPS-motivation}, classes that share more common ancestors (like $\{D6, \ldots, D10\}$) are inherently harder to distinguish and require greater learning capacity than simpler classes (like $\{A2, A3\}$). Current models fail to provide this adaptive capacity. The ideal solution must be hierarchically consistent and offer class-specific adaptive capacity to effectively separate all the classes.

Building upon the strength of neural networks to learn highly discriminative features, we propose that these unrelated yet discriminative features can be mapped to a vector space that is inherently hierarchy-aware, hierarchically consistent and can automatically adapt the learning capacity for complex classes. To achieve this, we first formally introduce a novel concept of Hierarchy-Aware Vector Spaces (HAVS). We then propose a framework called Hier-COS to construct such an HAVS $V_\mathcal{T}$ using fixed orthonormal bases, each tied to a class in the label hierarchy. This enables flexible subspace compositions ($V_i \subset V_\mathcal{T}$) to capture the diverse features of complex superclasses, while orthogonal bases learn nuanced, discriminative features specific to each class. We demonstrate that standard neural backbones can be augmented to map features into $V_\mathcal{T}$ using a lightweight transformation module, which reduces the severity of mistakes and promotes hierarchical consistency in top-$k$ predictions. While top-1 accuracy is important, we argue that top-$k$ predictions better reflect hierarchy-aware representations, thereby proposing a preference-based metric for a more faithful evaluation. We list our contributions below.

\begin{enumerate}
    \item We define Hierarchy-Aware Vector Spaces (HAVS) that are induced by a given hierarchy tree $\mathcal{T}$ and a tree distance function $D_\mathcal{T}$. We then propose Hierarchical Composition of Orthogonal Subspaces (Hier-COS), which is theoretically guaranteed to be an HAVS, and is therefore implicitly aligned with the structure of the hierarchy tree. 
    \item To the best of our knowledge, this is the first framework allowing unified \textit{`hierarchy-aware multi-class'} and \textit{`hierarchical multi-level'} classification. We demonstrate that a lightweight transformation module can learn the mapping from any vector space to Hier-COS, improving the severity of mistakes and hierarchical consistency.
    \item Hier-COS implicitly adapts the learning capacity for different classes based on their inherent classification complexity --- a vital property absent in existing methods. 
    \item We identify key limitations in hierarchical multi-class classification metrics commonly used in vision, and critique the applicability of ranking-based metrics from the information retrieval literature \cite{ndcg, ir-survey, dezert2024distancespartialpreferenceorderings, mnr-ndcg}. To address these gaps, we introduce Hierarchically Ordered Preference Score (HOPS)—a novel metric grounded in preference ordering \cite{dezert2024distancespartialpreferenceorderings}. It accounts for the top-1 classification performance as well as the severity of mistakes, thus eliminating the need for multi-metric evaluation.
    
    \item We evaluate our method on four hierarchical datasets: FGVC-Aircraft (3 levels) \cite{fgvc-aircraft}, CIFAR-100 (5 levels) \cite{krizhevsky_2009_CIFAR100}, iNaturalist-19 (7 levels) \cite{van_inat_2018_CVPR} and tieredImageNet (12 levels) \cite{ren_tiered_18_ICLR}. We achieve state-of-the-art (SOTA) on three datasets and competitive performance on tieredImageNet-H.
\end{enumerate}

%% file: sec/lit_review.tex
\section{Related Work}\label{sec:rel_work}



Learning representations that use the class hierarchy has been shown to directly correlate with making better mistakes \cite{bertinetto_2020_CVPR, baseline-hafeat} and robustness \cite{review-helus2022robustness}. In recent years, several works have exploited available taxonomy, not only for standard classification tasks \cite{flamingo_2021_Chang, frome_2013_NIPS} but also for image retrieval \cite{review-ramzi2022retrieval, barz_2019_WACV}, open set recognition \cite{review-hosr, review-lang2024osr, review-ancestorsearch-hosr}, object detection \cite{Feng_CSVT2025} and neural collapse \cite{baseline-Liang_2023-haframe}. 

Label embedding methods map ground truth classes to label embeddings in vector spaces \cite{bengio_2010_NIPS, frome_2013_NIPS, barz_2019_WACV, bertinetto_2020_CVPR} or hyperbolic spaces \cite{liu_2020_CVPR}, where the relative positions of the label embeddings represent their semantic similarity. These methods maximize the similarity between the feature vector and fixed label embeddings. Although hyperbolic embeddings \cite{nickel_2018_PMLR, liu_2020_CVPR, khrulkov_2020_CVPR} also capture hierarchical relations via non-Euclidean geometry, they require manifold optimization and have not been demonstrated on deep, fine-grained visual hierarchical classification. 

Hierarchy-aware features can also be learned by specialized loss functions that directly incorporate hierarchical information \cite{deng_2014_HEX_ECCV, bertinetto_2020_CVPR, 10.1007/978-3-031-70359-1_15}. 
Classification models employing architectural adaptations have also been proposed. Wu et al. \cite{wu_2016_ACM} developed a multi-task loss function where independent classification heads are used for classes at different granularity, i.e., levels of the label hierarchy. All the classifiers share a joint feature extractor and the cross-entropy losses of all levels are optimized jointly. Similarly, Chang et al. \cite{flamingo_2021_Chang} also propose to train additional classifiers for each level of granularity, although while disentangling the feature vectors for each level of granularity. This allows the learning of coarse and fine-grained features; however, it is limited to small hierarchies because of explicit partitioning of the feature space. To mitigate this, Garg et al. \cite{baseline-hafeat} utilize the shared feature extractor and use a cross-entropy loss at the finest level and aim to achieve hierarchical consistency by introducing a soft hierarchical consistency, geometric consistency and margin loss. All these methods rely on training additional classifiers and are paired with additional loss terms, making them harder to train, more resource-intensive, and more difficult to scale to larger hierarchies. Moreover, with independent classifiers, additional constraints are needed to maintain hierarchical consistency.

HAFrame \cite{baseline-Liang_2023-haframe} trains a single classifier while freezing the weight vectors to a hierarchy-aware frame --- a frame where semantically closer weight vectors have lower angular separation. Building on neural collapse theory \cite{papyan2020prevalence}, they map deep features to this fixed output frame via a learnable transformation. While HAFrame improves accuracy and metrics for severity of mistakes, it inadequately captures the hierarchy structures as shown in Fig.~\ref{fig:teaser}. 

To the best of our knowledge, there exists no single model that can simultaneously perform \textit{`hierarchy-aware multi-class'} and \textit{`hierarchical multi-level'} classification. Moreover, the learned feature representations are confined to live in the direction of the weight vectors, thereby restricting the inter-class separability due to smaller angular separation. Finally, none of the existing methods can automatically adapt the learning capacity for the classes based on the hierarchy tree.

%% file: sec/methodology.tex
\section{Hierarchy-Aware Vector Spaces (HAVS)}\label{sec:hier_cos}

\subsection{Problem Setup and Notation} 
We consider a dataset $\calX$ and a set of class labels $\mathcal{Y}$. We are given the semantic relationship between the classes $\mathcal{Y}$ in the form of a hierarchy tree $\calT$ comprising of the set of nodes, $\calV_\calT\! =\!\{v_1, \ldots, v_n\}\cup\{v_0\}$, where $v_0$ denotes the root node representing the universal class and is never used. The tree $\calT$ has height $H$ with the number of nodes at each level given by $\{K_1,\ldots,K_H\}$. We separately define the set of $K$ leaf nodes as $\calV_{\ell}=\{v_{y_1}, \ldots, v_{y_{K}}\}\subset\calV_\calT$ to accommodate for leaf nodes at any height, a property often observed in naturally occurring label hierarchies. Here, the indices $y_1,\ldots,y_K$ are the fine-grained class labels $\mathcal{Y}$. For each node $v_i\in\calV_\calT$, we define the set of its ancestors (excluding $v_0$) as $f_a(v_i)$ and the set of its descendants as $f_d(v_i)$. Note that since each intermediate node in the tree is also a class label (super-category of a fine-grained class), we will interchangeably use the term \textit{node} or \textit{class}.

\begin{definition} \label{def:hafs}
    Consider a vector space $V_\mathcal{H}$ and a hierarchy tree $\mathcal{T}$. Let $\mathbf{x} \in V_\mathcal{H}$ be a non-zero feature vector representing the input $x$ belonging to class $y_i$. Assume $\{V_{y_1}, \ldots, V_{y_K}\}$ to be the subspaces of $V_\mathcal{H}$ such that
    \[\mathbf{x} \in V_{y_j} \iff y_j = y_i\]
    Let's denote the distance of feature vector $\mathbf{x}$ from the subspace $V_{y_i}$ as $D_i$ and the tree distance between classes $y_i$ and $y_j$ as $D_\mathcal{T}(y_i, y_j)$. The vector space $V_\mathcal{H}$ is said to be a \framebox[1.01\width]{\textsc{Hierarchy-Aware Vector Space}} induced by the tree $\mathcal{T}$ and a tree distance function $D_\mathcal{T}$ if and only if the following condition holds:
    \begin{align} \label{eq:hafs}
        \begin{split}
            \text{if} \hspace{15pt} D_\mathcal{T}(y_i, y_j) &< D_\mathcal{T}(y_i, y_k) \\
            \text{then} \hspace{10pt} \vert D_i - D_j \vert &< \vert D_i - D_k \vert \\
            &\hspace{4pt}\forall \hspace{2pt} (y_i, y_j, y_k) \in \mathcal{Y}
        \end{split}
    \end{align}
\end{definition}

Note that $D_\mathcal{T}(y_i, y_j)$ could be any valid tree distance function that computes the distance between any two nodes $y_i$ and $y_j$ in the tree, for instance, LCA-based or hop-based. This tree distance function is used to define the \textit{desired partial preference order} that we want the learned vector space to inherit. On the other hand, the distance $D_i$ between a feature vector $\mathbf{x}$ and a subspace $V_{y_i}$ is always defined as the distance between the vector and its orthogonal projection onto $V_{y_i}$, discussed later in this section. Since the LCA-based tree distance function provides a more nuanced measure of semantic similarity by considering the shared ancestry of the classes, it has been widely adopted to define the desired partial preference order over classes \cite{barz_2019_WACV, bertinetto_2020_CVPR, flamingo_2021_Chang, baseline-hafeat, baseline-Liang_2023-haframe, park2024learninghierarchicalsemanticclassification}.

\subsection{Hierarchical Composition of Orthogonal Subspaces (Hier-COS)}
In this section, we focus on defining an HAVS that is induced by $D_\mathcal{T}$, where $D_\mathcal{T}$ is the LCA-based tree distance function. It is important to note that while there could exist multiple HAVS that are induced by $D_\mathcal{T}$, we aim at defining an HAVS that is inherently able to adapt the learning capacity for each class, and is also capable of performing hierarchical multi-level classification. Therefore, instead of defining a subspace for only the fine-grained classes, we define it for all the classes in $\mathcal{T}$ to permit the generalization of HAVS to any level of the hierarchy tree.

We define a vector space $V_\mathcal{T}$ using an orthonormal bases $\calE=\{e_1,\ldots,e_n\}$ of $\bbR^n$. Without loss of generality, we create an arbitrary bijective assignment of basis vectors $e_i\in\calE$ to nodes $v_i\in\calV_\calT$. We define $\mathcal{E}^{(l)}$ as the set of basis vectors corresponding to the nodes at level $l$. For a given node $v_i\in\calV_\calT$, we also define $\calE^a_i$ as the basis set corresponding to $f_a(v_i)$ and $\calE^d_i$ as the set of basis corresponding to $f_d(v_i)$. For defining the vector space corresponding to a node $v_i$, we define the basis set as $\calE_i=\calE^a_i\cup \{e_i\}\cup\calE^d_i$,  
and  the corresponding subspace as $V_i=\text{span}(\calE_i), \forall i\in [n]$, where $[n] = \{k \mid k \in \mathbb{Z}^+, k \leq n\}$. Finally, the vector space $V_\calT=\text{span}(\calE)$\footnote{A summary of notations is given in the supplementary.}. In Theorem \ref{theorem:hiercos} (proof in supplementary), we observe that \ul{\textit{Hier-COS will be an HAVS if the bases are orthogonal and the feature vectors in a subspace span all its dimensions}}, thus motivating our design choices.
A direct implication of Theorem \ref{theorem:hiercos} is Proposition \ref{proposition:consistency} (proof in supplementary), stating that the label hierarchy predicted using Hier-COS is guaranteed to be hierarchically consistent.

\begin{theorem} \label{theorem:hiercos}
    Consider a vector space $V_\mathcal{T}$ and its subspaces for fine-grained classes $\{V_{y_1}, \cdots, V_{y_K}\}$ spanned by $\{\mathcal{E}_{y_1}, \cdots, \mathcal{E}_{y_K}\}$ defined using the hierarchy tree $\mathcal{T}$, as discussed above. For all $\mathbf{x} \in V_\mathcal{H}$, if $\mathbf{x} \in V_{y_i}$ and $\langle \mathbf{x}, e_{i} \rangle^2 > 0, ~\forall~  e_{i} \in \mathcal{E}_{y_i}$, then $V_\mathcal{H}$ is an $n$-dimensional HAVS induced by an LCA-based tree distance function $D_\mathcal{T}$.
\end{theorem}

\begin{proposition} \label{proposition:consistency}
    The label hierarchy predicted using Hier-COS is consistent across all the levels, i.e.,  $\{\hat{y}^{(1)}, \ldots, \hat{y}^{(l)}, \ldots, \hat{y}^{(H)}\}$ is a valid path in the tree.
\end{proposition}

\subsection{Learning Problem Formulation}

To learn feature representations in $V_\mathcal{T}$, we begin with a backbone feature extractor $f_\Theta:\calX\rightarrow Z\subseteq\bbR^d$ that has sufficient capacity to perform flat classification over the label set $\mathcal{Y}$. Our goal is to use this feature extractor as a backbone and learn a non-linear mapping $f_\theta: Z\rightarrow V_\calT$ from the feature space $Z$ to a vector space $V_\mathcal{T} = \mathbb{R}^n$, i.e., $\mathbf{x} = [x_1, \ldots, x_n]^\intercal= f_\theta(\mathbf{z}), ~\mathbf{z} \in Z$. Note that the orthonormal vectors $e_i\in\calE$ serve as canonical basis for $\bx\in V_\calT$, i.e.,  $|x_i|, i\in[n]$ are the components along the basis vector $e_i$.

Based on this construction, we define the distance between a point and the subspace using orthogonal projections. Consider a vector $\bx\in V_\calT$ and a subspace $V_{y_j}\subset V_\calT$. The squared distance between $\bx$ and $V_{y_j}$ is computed as the distance between the point $\bx$ and its orthogonal projection onto $V_{y_j}$, computed using the projection operator $\bbP_{\calE_{y_j}}$, i.e., 
\begin{align}
    d_S^2(\bx,V_{y_j}) \!= \!\Vert\bbP_{\neg\calE_{y_j}} \bx\Vert^2 
    \!= \Vert\bx_{V_{y_j}^\perp}\Vert^2 =\! \sum_{e\in \neg\calE_{y_j}} x_e^2 
    \label{eqn:sub_dist}
\end{align}
where, $\lVert \cdot \rVert$ is the L-2 norm operator, and with a slight abuse of notation, we define $\neg\calE_{y_j}:=\{\calE\setminus\calE_{y_j}\}$ that serves as the set of basis vectors corresponding to the orthogonal complement of $V_{y_j}$, i.e., $V_{y_j}^\perp$. We also define $\bx_{V_{y_j}}$ and $\bx_{V_{y_j}^\perp}$ as projections of $\bx$ onto $V_{y_j}$ and $V_{y_j}^\perp$, respectively, such that $\bx = \bx_{V_{y_j}} + \bx_{V_{y_j}^\perp}$. Now, consider a vector $\bx$ that corresponds to an \emph{ideal sample} from class $y_i$, such that $\bx\in V_{y_i}$, i.e., $ \Vert\bx_{V_{y_i}^\perp}\Vert\!=\!0$. If the leaf classes $y_i$ and $y_j$ have a common parent, i.e., $LCA(y_i, y_j)\!=\!1$, from equation (\ref{eqn:sub_dist}), it is easy to see that the distance $d_S(\bx,V_{y_j}) = \vert x_{e_{y_i}}\vert$ because $\vert x_e \vert = 0 ~\forall e \in \{\neg\calE_{y_j} \setminus e_{y_i}\}$. Generalizing this idea to an arbitrary class $y_k$, the distance $d_S(\bx,V_{y_k})$ for any $\bx\in V_{y_i}$ is dependent only on the components $\vert x_e\vert,~e\in \{\calE_{y_i} \setminus \calE_{y_k}\}$. If the projection norm of the feature vector is concentrated entirely on the leaf, i.e. $\vert x_e \vert = 0, ~ \forall e \in \{\mathcal{E} \setminus e_{y_i}\}$, the distance between all the classes would be equal, i.e. $\vert x_{e_{y_i}} \vert$, defaulting to \emph{flat} classification. Contrarily, if the projection norm is concentrated entirely on the ancestors $(\calE^a_{y_i})$ and $\vert x_{e_{y_i}}\vert=0$, the leaf classes will be inseparable.  
Accordingly, we propose to distribute the magnitude of $\bx$ across the basis $e\in\calE_{y_i}$ to promote high leaf-level discrimination while respecting the semantic similarity via coarser classes. Therefore, we use a monotonically increasing weight function to distribute the magnitude of $\bx$ as we move from the basis vectors of the root node to the leaf. Furthermore, we desire $\vert x_e\vert=0,~\forall e\in\neg\calE_{y_i}$, resulting in a sparse vector $\bx$. 

Given the strength of modern backbones, we adopt the minimal complexity transformation module from HAFrame \cite{baseline-Liang_2023-haframe}, which serves as a mapping layer to project backbone features into Hier-COS. The model can be trained either end-to-end or by freezing the backbone and training only the transformation module, using the loss defined below.

\subsection{Loss Function}
Given a feature vector $\mathbf{x} = [x_1, \ldots, x_n]^\intercal \in V_\mathcal{T}$, our objective is to learn the distribution of $\vert x_i \vert ~\forall i \in [n]$. Note that we do not assume all the leaf nodes to be at the same level, thereby denoting the level of a leaf node as $h \leq H$. Usually, since $H\ll n$, we implicitly want to learn a mapping such that $\bx$ is sparse with only $\vert x_e\vert\neq 0,~e\in\calE_{y_i}$. For the weight distribution of $\bx$ across the basis $\calE_{y_i}$, we use a variant of the Tree path KL-divergence loss \cite{park2024learninghierarchicalsemanticclassification}. We define a target distribution $P$ over all the $n$ nodes by concatenating the weighted one-hot encoded vectors from all the hierarchical levels. Specifically, we define the weight distribution by the $n$-dimensional vector $\bw=[w_1\bdelta_1,\ldots,w_h\bdelta_h,\ldots,w_H\bdelta_H]^\intercal$, where $\bdelta_l$ represents a one-hot-encoded vector for the hierarchy level $l \leq h$, otherwise a zero vector. 
The weights $w_l$ for $l \in [h]$ are generated by the function $w_l=\exp\left(\frac{1}{h+1-l}\right)$. This exponential mapping ensures the sequence is strictly monotonically increasing and exhibits a heavy-tailed characteristic, where the influence of the terminal weights ($w_h$) is amplified. Crucially, the calculation of $w_l$ involves an exponent that remains positive and bounded ($\frac{1}{h} \!\le\! \frac{1}{h+1-l} \!\le\! 1$), which is deliberately designed to maintain {numerical stability.
The target distribution $P$ is then obtained by normalizing the weight vector $\bw$.
Our first loss term is the KL divergence between the target ($P$) and the predicted ($\widehat{P}$) weighted tree path distributions:
\begin{align}
    \mathcal{L}_{kl} = \text{KL}(P||\widehat{P}) = \sum_{i\in[n]} P_i~\log\frac{P_i}{\widehat{P}_i}
\end{align}
where, $\widehat{P}_i = \frac{e^{\vert \mathbf{x}_i \vert}}{\sum_{j=1}^{n} e^{\vert \mathbf{x}_j \vert}}$ ~ $\forall i = [n]$ is computed using the softmax function to ensure that the distance between the correct and incorrect classes is relatively high. As discussed before, the class subspaces only span at most $H$ dimensions and therefore require $\bx$ to be sparse. Specifically, a leaf node at level $h$ is only associated with $h$ hierarchical class labels, implying that the feature vectors must span exactly one dimension for levels $l \leq h$ and zero dimensions for levels $h < l \leq H$. Therefore, we use a regularization term:
\begin{align}
     \mathcal{L}_{reg} = \sum_{l=1}^{h} \bigg\lVert \boldsymbol\delta_l - \frac{\mathbb{P}_{\mathcal{E}^{(l)}}\mathbf{x}}{\lVert \mathbb{P}_{\mathcal{E}^{(l)}}\mathbf{x} \rVert} \bigg\rVert_1 + \sum_{l=h}^H \bigg\lVert \frac{\mathbb{P}_{\mathcal{E}^{(l)}}\mathbf{x}}{\lVert \mathbb{P}_{\mathcal{E}^{(l)}}\mathbf{x} \rVert}  \bigg\rVert_1
\end{align}
where, $\lVert \cdot \rVert_1$ represents L1-norm.
The final training loss for each sample is given by $\calL_{total} = \calL_{kl} + \alpha \calL_{reg}$.
During inference, we simply pick the class using $\widehat{y}=\arg\max_{y_i\in\calV_\ell} \lVert\bbP_{\mathcal{E}_{y_i}}\bx \rVert$.

%% file: sec/hops_method.tex
\section{Hierarchical Evaluation Metrics}\label{sec:hops}

\begin{figure*}[ht!]
    \centering
    \includegraphics[width=0.8\linewidth]{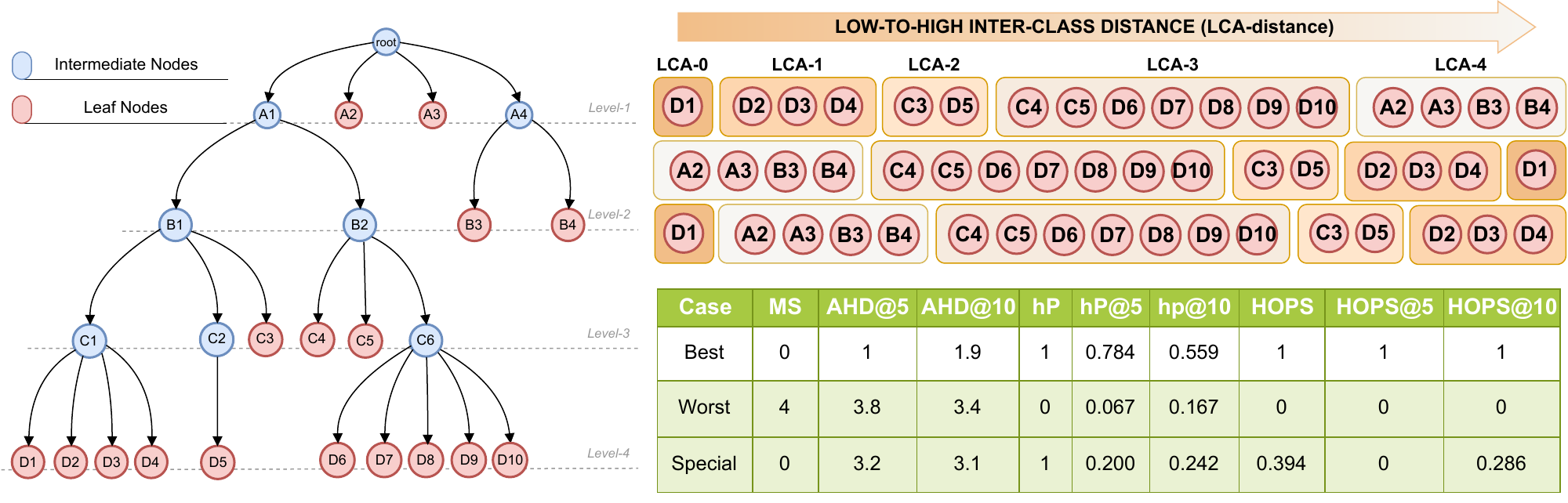}
    \caption{Best Case: The ground truth and predicted preference ordering is identical. Worst Case: The reverse preference ordering is obtained. Special Case: The top-1 prediction is correct and the following  predictions are in reverse preference order.}
    \label{fig:HOPS-motivation}
\end{figure*}

Limitations of existing metrics have also been pointed out before \cite{baseline-hafeat}. We highlight the key gaps and propose a novel, partial preference order based metric and defer a detailed discussion to the supplementary. 

\subsection{Overview of Shortcomings}

The LCA-based tree distance function has been widely adopted for defining the desired partial preference order over classes, owing to its ability to provide a more nuanced measure of semantic similarity, or equivalently, severity of mistakes, through the consideration of shared class ancestry.
Mistake Severity (MS) \cite{bertinetto_2020_CVPR} reports the average LCA distance between the true and the predicted classes for \emph{all} the misclassified samples. \cite{karthik_2021_ICLR} points out the bias in MS that leads to \emph{better, but more} mistakes, and proposes Average Hierarchical Distance (AHD) that averages LCA distance over \emph{all} samples. However, this leads to a bias towards higher top-1 accuracy without adequately capturing the severity of mistakes \cite{baseline-hafeat}. Other metrics like hierarchical Precision (hP) and Recall (hR) have been used, but can be written as variants of a normalized version of AHD and thus suffer the same limitations. Moreover, directly using averaged LCA-based distances in the metric makes it dependent on the properties of the specific tree (e.g., height, imbalance and branching factor) and harder to interpret and compare. This difficulty is evident from the small relative differences across methods for a given dataset, and large differences across datasets while reporting MS, AHD and its $@k$ in experiments reported by us and others \cite{karthik_2021_ICLR,baseline-hafeat,baseline-Liang_2023-haframe}. 


We consider the hierarchy tree and the specific cases presented in Fig. \ref{fig:HOPS-motivation}, where $D1$ is the true class and three possible prediction cases are analysed. We see that the MS and AHD@$k$ are unnormalized and, therefore, difficult to interpret. Moreover, AHD@$k$ and hP@$k$ deteriorate even for the best case as $k$ increases. This also leads to a reduction of the dynamic range of the hP score, making it difficult to interpret despite being normalized. Another severe shortcoming of AHD@$k$ is that it computes the average LCA over the top-$k$ predictions. Since the average is permutation invariant, the AHD@5 for the best case in Fig. \ref{fig:HOPS-motivation} and a prediction instance with a much worse top-5 order as $\{C3,D4,D3,D2,D1\}$ will yield the exact same AHD@5 score. Fig. \ref{fig:teaser} also quantitatively highlights these limitations using large, real-world datasets. A detailed discussion on other evaluation metrics and their limitations, like TPO, MRR, MNR and NDCG \cite{dezert2024distancespartialpreferenceorderings, mnr-ndcg} is in the supplementary.

\subsection{Hierarchically Ordered Preference Score}
To address the limitations of existing hierarchical evaluation metrics, we introduce a preference order-based metric, Hierarchically Ordered Preference Score (HOPS). Due to page limit constraints, we provide an intuitive description of HOPS and its variant HOPS$@k$, and defer the technical details to the supplementary. In Fig.~\ref{fig:teaser}, we compared the percentage of samples with a correct prediction order, which can be used to quantify the hierarchical performance. However, similar to the accuracy metric, it is more strict as it considers the prediction order to either be correct or incorrect, disregarding the severity of mistakes. To address this, for each target class, we associate a desired rank to each class based on their partial preference order to obtain a $K$-dimensional vector representation $z=[z_1, \ldots, z_K]$, wherein same ranks are given to classes that are at the same LCA distance. During inference, we get the order of classes based on their prediction probabilities, and map them to their desired ranks using $z$ to obtain a vector representation $\widehat{z} = [\widehat{z}_1 \ldots, \widehat{z}_K]$. To compute HOPS, we simply calculate the weighted absolute difference between the vector representations of the desired and the predicted preference orders
\begin{align} \label{eq:hops}
    \text{HOPS} = 1- \frac{s}{s_{max}} =  1 - \frac{\sum_{j=1}^{K} \eta_j \cdot |z_j - \hat{z_j}|}{\sum_{j=1}^{K} \eta_j \cdot |z_j - z_j^*|}
\end{align}
where $z^* = [z_1^*, \ldots, z_K^*] = [z_{K:-1:1}]$\footnote{In $z_{a:c:b}$, $a$ is the start, $b$ is the stop and $c$ is the step-size} is the rank for the worst prediction order, i.e., the reverse of $z$, and $\eta_j$'s are defined using a multi-step exponential-linear decay function. 

We are often interested in only the top-$k$ performance of a classifier, particularly, when there are a large number of classes. We define the HOPS@$k$ variants, simply by modifying the notion of the worst achievable score and adapting the weights. For HOPS@$k$, we define the worst ranked vector $z_k^* = [z_{k:-1:1},z_{k+1:1:K}]$ yielding the score as $s^k_{max}$ which is then used to compute HOPS@$k$: 
\begin{equation}
    \text{HOPS}@k = \max\left(0,1-\frac{s^k}{s^k_{max}}\right)
\end{equation}
where $s^k$ is computed using (\ref{eq:hops}), but with $\eta_{j>k}=0$. 
Classifier performance is reported using scores averaged over all samples. We emphasize that for $k=1$, the HOPS@$k$ is equal to the top-1 accuracy. This is aligned with the design of this metric, where the HOPS@$k$ concerns with the ranking performance of up to $k^{th}$ ranked prediction only. 

%% file: sec/results.tex
\section{Experiments, Results and Analysis}\label{sec:expt}
In this section, we summarize the results on FGVC-Aircraft \cite{fgvc-aircraft}, CIFAR-100 \cite{krizhevsky_2009_CIFAR100} and iNaturalist-19 \cite{van_inat_2018_CVPR} in Tables \ref{tab:hfgvc-fgvc}, \ref{tab:hfgvc-cifar} and \ref{tab:hfgvc-inat}, respectively. We report and analyze the results on tieredImageNet-H separately in the supplementary section. For methods denoted with an asterisk (*), we simply use the metrics reported by \cite{baseline-hafeat}. For each metric (column) in these tables, we highlight the cells using a color scale, where `white' represents the bottom 50 percentile values and \colorbox[HTML]{b6d7a8}{`green'} represents the best metric. Details on the dataset, baseline methods, backbones used, evaluation metrics, implementation details, training configuration and ablations are provided in the supplementary.

We observe an improvement of $1.2\%$, $0.26\%$ and $0.68\%$ in the top-1 accuracy using our approach for FGVC-Aircraft, CIFAR-100 and iNaturalist-19, respectively, as compared to HAFrame. The improvement in the case of ViT is $2.42\%$, which is a significant gain in comparison to the baseline. We also observe that, except for FGVC-Aircraft, we are competitive or better than SOTA in terms of MS across all the experiments. A noteworthy observation is that $K=100$ and $n=200$ for FGVC-Aircraft. This implies that a lot of intermediate nodes have only one child, which is also evident from the number of classes having the smallest LCA distance $(1, 2, 3)$ are $(46, 41, 13)$, respectively. We point out that the maximum LCA distance for FGVC-Aircraft is $3$, and metrics like MS and AHD, when averaged over all samples, generate biased measure of mistake severity, as also discussed in \cite{baseline-hafeat}. Moreover, MS only considers the average cost of mistakes, and thus is reported on different numbers of samples. Hence, it needs to be compared only when coupled with the top-1 accuracy. This is also evident for the `Soft-labels $\beta=4$' and `Barz \& Denzler' methods in Tables \ref{tab:hfgvc-cifar} and \ref{tab:hfgvc-inat}, respectively. We also observe that the reported AHD@$k$ metric for Hier-COS is consistently improved over all the baseline methods. 

Considering multiple metrics like top-1 accuracy and MS or AHD@k for evaluation makes comparisons ambiguous. For instance, Flamingo-I in Table \ref{tab:hfgvc-cifar}, as compared to HAFeat-I, has an improved top-1 accuracy and AHD@\{5, 20\} but a worse MS and AHD@1. Thus, it is unclear how to quantify which method is better. HOPS addresses this issue by considering not only the top-1 accuracy and severity of mistakes but also the preference order and structure of the hierarchy tree. We note that Hier-COS does consistently better w.r.t. HOPS across all experimental settings. Particularly relevant is the significant improvement over HAFrame for iNaturalist-19 because of the higher number of classes, thus highlighting the proficiency of Hier-COS in learning stronger hierarchical representations. As the ordering of \emph{all} predictions may not be relevant, we observe that Hier-COS performs consistently better for HOPS@\{5,20\}. 

We further validate the richness of the hierarchy-aware feature representations by comparing the order of predictions in Fig \ref{fig:teaser}. We observe that the percentage of samples with incorrect order drastically reduces for all the baselines across all the datasets. However, Hier-COS is largely able to maintain its predictions, even for higher values of $k$. 


Lastly, we report the FPA and HOPS on FGVC-Aircraft, CIFAR-100 and iNaturalist-19 in Table \ref{tab:hier-consistency} along with the top-1 accuracy. We observe that the FPA is improved by $3.64\%$ for FGVC-Aircraft, by $1.36\%$ for CIFAR-100 and by $1.51\%$ for iNaturalist-19, as compared to the best method, with significant gains in HOPS. We also observe that the drop in top-1 accuracy and FPA, i.e. (Accuracy - FPA), is the least in Hier-COS, thus confirming its ability to promote hierarchical consistency. We further highlight that all the previous methods impose explicit constraints in the loss function to achieve hierarchical consistency, however, our framework implicitly promotes consistency and does not require additional constraints. We also report the level-wise accuracies in the supplementary 
and report that we perform better or competitive compared to all the baselines. Additionally, although Cross-Entropy and HAFrame do not explicitly allow classification across all the hierarchical levels, we compute the probabilities for a super-class as the sum of probabilities of all its children \cite{baseline-hafeat} to compare the FPA on CIFAR-100. The FPA for Cross-Entropy and HAFrame is 77.11 $\pm$ 0.002 and 77.0 $\pm$ 0.002, respectively, which is lower compared to Hier-COS.

\begin{table*}[!htp]\centering
\scriptsize
\resizebox{\textwidth}{!}{
    \begin{tabular}{l|c|c|c|c|c|c|c|c}\toprule
    \textbf{Method} & \textbf{Accuracy $(\uparrow)$} & \textbf{MS $(\downarrow)$} & \textbf{AHD@1 $(\downarrow)$} & \textbf{AHD@5 $(\downarrow)$} & \textbf{AHD@20 $(\downarrow)$} & \textbf{HOPS $(\uparrow)$} & \textbf{HOPS@5 $(\uparrow)$} & \textbf{HOPS@20 $(\uparrow)$} \\\midrule
    Cross Entropy &79.57 $\pm$ 0.731
    &2.09 $\pm$ 0.011
    &0.43 $\pm$ 0.014
    &2.10 $\pm$ 0.004
    &2.67 $\pm$ 0.001
    &0.56 $\pm$ 0.003&0.51 $\pm$ 0.005&0.57 $\pm$ 0.003
    \\
    Flamingo-I&\cellcolor[HTML]{d3e7ca} 80.71 $\pm$ 0.408&\cellcolor[HTML]{e3f0de} 2.07 $\pm$ 0.039&0.40 $\pm$ 0.013
    &2.06 $\pm$ 0.005
    &2.65 $\pm$ 0.001
    & 0.58 $\pm$ 0.004& 0.54 $\pm$ 0.003& 0.59 $\pm$ 0.003
    \\
    Flamingo-II&\cellcolor[HTML]{f7fbf5} 80.07 $\pm$ 0.364&2.10 $\pm$ 0.014
    &0.42 $\pm$ 0.008
    &2.08 $\pm$ 0.007
    &2.66 $\pm$ 0.001
    & 0.57 $\pm$ 0.005& 0.52 $\pm$ 0.005& 0.58 $\pm$ 0.004
    \\
    HAFeat-I&73.30 $\pm$ 0.361
    &2.50 $\pm$ 0.074
    &0.67 $\pm$ 0.027
    &2.10 $\pm$ 0.010
    &2.61 $\pm$ 0.008
    & 0.53 $\pm$ 0.002& 0.45 $\pm$ 0.002& 0.53 $\pm$ 0.002
    \\
    HAFeat-II&74.53 $\pm$ 0.592
    &2.53 $\pm$ 0.033
    &0.65 $\pm$ 0.022
    &2.10 $\pm$ 0.006
    &2.61 $\pm$ 0.004
    & 0.53 $\pm$ 0.009& 0.45 $\pm$ 0.012& 0.53 $\pm$ 0.009
    \\
    HAFrame &\cellcolor[HTML]{e3f0de}80.55 $\pm$ 0.346&\cellcolor[HTML]{b6d7a8}2.00 $\pm$ 0.029&\cellcolor[HTML]{c6dfbb}0.39 $\pm$ 0.009&\cellcolor[HTML]{b8d8aa}1.74 $\pm$ 0.006&\cellcolor[HTML]{bad9ad}2.45 $\pm$ 0.003&\cellcolor[HTML]{dbebd4} 0.86 $\pm$ 0.005&\cellcolor[HTML]{dbebd4}0.81 $\pm$ 0.005&\cellcolor[HTML]{dbebd4}0.86 $\pm$ 0.005
    \\
    Hier-COS &\cellcolor[HTML]{b6d7a8}81.75 $\pm$ 1.006&2.09 $\pm$ 0.029&\cellcolor[HTML]{b6d7a8}0.38 $\pm$ 0.020&\cellcolor[HTML]{b6d7a8}1.73 $\pm$ 0.009&\cellcolor[HTML]{b6d7a8}2.44 $\pm$ 0.003&\cellcolor[HTML]{b6d7a8}0.89 $\pm$ 0.007&\cellcolor[HTML]{b6d7a8}0.84 $\pm$ 0.010&\cellcolor[HTML]{b6d7a8}0.88 $\pm$ 0.007\\
    \bottomrule
    \end{tabular}
}
\caption{Results comparing the performance of hierarchical classification on the test set of FGVC-Aircraft.}\label{tab:hfgvc-fgvc}
\end{table*}

\begin{table*}[!htp]\centering
\scriptsize
\resizebox{\textwidth}{!}{
    \begin{tabular}{l|c|c|c|c|c|c|c|c}\toprule
    \textbf{Method} & \textbf{Accuracy $(\uparrow)$} & \textbf{MS $(\downarrow)$} & \textbf{AHD@1 $(\downarrow)$} & \textbf{AHD@5 $(\downarrow)$} & \textbf{AHD@20 $(\downarrow)$} & \textbf{HOPS $(\uparrow)$} & \textbf{HOPS@5 $(\uparrow)$} & \textbf{HOPS@20 $(\uparrow)$} \\\midrule
    Cross Entropy &\cellcolor[HTML]{bddbb1}77.77 $\pm$ 0.248&2.33 $\pm$ 0.014
    &\cellcolor[HTML]{e6f1e2}0.52 $\pm$ 0.007&2.25 $\pm$ 0.007&3.19 $\pm$ 0.004&0.54 $\pm$ 0.001&0.05 $\pm$ 0.001&0.17 $\pm$ 0.001
    \\
    Barz \& Denzler* &68.31 $\pm$ 0.004
    &2.36 $\pm$ 0.025
    &0.75 $\pm$ 0.012
    &\cellcolor[HTML]{d8e9d0}1.25 $\pm$ 0.364&\cellcolor[HTML]{e0eeda}2.49 $\pm$ 0.004&-&-&-
    \\
    YOLO-v2* &67.97 $\pm$ 0.006&3.72 $\pm$ 0.022&1.19 $\pm$ 0.019&2.85 $\pm$ 0.010&3.39 $\pm$ 0.011&-&-&-
    \\
    HXE+CRM* &\cellcolor[HTML]{c7e1bc}77.63 $\pm$ 0.280&\cellcolor[HTML]{fcfdfb}2.30 $\pm$ 0.026&\cellcolor[HTML]{cee4c5} 0.51 $\pm$ 0.009&\cellcolor[HTML]{b6d7a8} 1.11 $\pm$ 0.008&\cellcolor[HTML]{b6d7a8}2.18 $\pm$ 0.003&-&-&-
    \\
    Soft-labels ($\beta$=30)* &73.01 $\pm$ 0.003
    &2.38 $\pm$ 0.004&0.64 $\pm$ 0.008
    &\cellcolor[HTML]{fafcf9}1.39 $\pm$ 0.027&2.79 $\pm$ 0.005&-&-&-
    \\
    Soft-labels ($\beta$=4)* &67.85 $\pm$ 0.008&\cellcolor[HTML]{c7e0bc}2.21 $\pm$ 0.037&0.71 $\pm$ 0.024&\cellcolor[HTML]{d3e6ca}1.23 $\pm$ 0.018&\cellcolor[HTML]{bcdab0}2.23 $\pm$ 0.008&-&-&-
    \\
    Flamingo-I &\cellcolor[HTML]{b6d7a8}77.87 $\pm$ 0.137&2.31 $\pm$ 0.026
    &\cellcolor[HTML]{cee4c5}0.51 $\pm$ 0.007&2.06 $\pm$ 0.016
    &3.08 $\pm$ 0.006
    &0.57 $\pm$ 0.002&0.09 $\pm$ 0.004&0.21 $\pm$ 0.004
    \\
    Flamingo-II &75.26 $\pm$ 0.508&2.32 $\pm$ 0.029&0.57 $\pm$ 0.017&2.41 $\pm$ 0.007
    &3.29 $\pm$ 0.017&0.51 $\pm$ 0.003&0.03 $\pm$ 0.001&0.14 $\pm$ 0.004\\
    HAFeat-I &\cellcolor[HTML]{cfe5c6}77.51 $\pm$ 0.224&\cellcolor[HTML]{d9ead1}2.24 $\pm$ 0.025&\cellcolor[HTML]{b6d7a8}0.50 $\pm$ 0.006&1.43 $\pm$ 0.009
    &\cellcolor[HTML]{f4f9f2}2.64 $\pm$ 0.003&\cellcolor[HTML]{f4f9f2}0.69 $\pm$ 0.001&\cellcolor[HTML]{f4f9f2}0.41 $\pm$ 0.007&\cellcolor[HTML]{f4f9f2}0.47 $\pm$ 0.004\\
    HAFeat-II &76.47 $\pm$ 0.219&\cellcolor[HTML]{eaf3e6}2.27 $\pm$ 0.019&0.54 $\pm$ 0.008&1.70 $\pm$ 0.012&2.84 $\pm$ 0.007&0.63 $\pm$ 0.001&0.22 $\pm$ 0.005&0.32 $\pm$ 0.001\\
    HAFrame &\cellcolor[HTML]{cee4c5}77.53 $\pm$ 0.211&\cellcolor[HTML]{d9ead1}2.24 $\pm$ 0.019&\cellcolor[HTML]{b6d7a8}0.50 $\pm$ 0.008&\cellcolor[HTML]{b8d8aa}1.12 $\pm$ 0.005&\cellcolor[HTML]{b6d7a8}2.18 $\pm$ 0.002&\cellcolor[HTML]{b6d7a8}0.92 $\pm$ 0.001&\cellcolor[HTML]{b6d7a8}0.72 $\pm$ 0.003&\cellcolor[HTML]{b6d7a8}0.83 $\pm$ 0.002\\
    Hier-COS&\cellcolor[HTML]{b6d7a8}77.79 $\pm$ 0.145&\cellcolor[HTML]{c7e0bc}2.21 $\pm$ 0.017&\cellcolor[HTML]{b6d7a8}0.49 $\pm$ 0.006&\cellcolor[HTML]{b6d7a8}1.09 $\pm$ 0.005&\cellcolor[HTML]{b6d7a8}2.17 $\pm$ 0.003&\cellcolor[HTML]{b6d7a8}0.93 $\pm$ 0.001&\cellcolor[HTML]{b6d7a8}0.76 $\pm$ 0.002&\cellcolor[HTML]{b6d7a8}0.85 $\pm$ 0.002\\
    \bottomrule
    \end{tabular}
}
\caption{Results comparing the performance of hierarchical classification on the test set of CIFAR-100.}\label{tab:hfgvc-cifar}
\end{table*}

\begin{table*}[!htp]\centering
\scriptsize
\resizebox{\textwidth}{!}{
    \begin{tabular}{l|c|c|c|c|c|c|c|c}\toprule
    \textbf{Method} & \textbf{Accuracy $(\uparrow)$} & \textbf{MS $(\downarrow)$} & \textbf{AHD@1 $(\downarrow)$} & \textbf{AHD@5 $(\downarrow)$} & \textbf{AHD@20 $(\downarrow)$} & \textbf{HOPS $(\uparrow)$} &\textbf{HOPS@5 $(\uparrow)$} &\textbf{HOPS@20 $(\uparrow)$} \\\midrule
    Cross Entropy &\cellcolor[HTML]{cfe5c5}70.86 $\pm$ 0.137&2.23 $\pm$ 0.012&0.65 $\pm$ 0.003&1.95 $\pm$ 0.007&3.37 $\pm$ 0.003&0.46 $\pm$ 0.001&0.28 $\pm$ 0.002&0.08 $\pm$ 0.001
    \\
    Barz \& Denzler* &37.37 $\pm$ 0.278 
    &\cellcolor[HTML]{b6d7a8}1.99 $\pm$ 0.008 &1.24 $\pm$ 0.005 
    &\cellcolor[HTML]{fdfefd}1.49 $\pm$ 0.005 &\cellcolor[HTML]{eaf3e6}1.97 $\pm$ 0.005&-&-&-
    \\
    YOLO-v2* &55.83 $\pm$ 0.046 &2.43 $\pm$ 0.001 
    &1.10 $\pm$ 0.001 &1.50 $\pm$ 0.001 &\cellcolor[HTML]{eef6eb}1.99 $\pm$ 0.002&-&-&-
    \\
    HXE+CRM* &\cellcolor[HTML]{e0eeda}70.67 $\pm$ 0.210&2.16 $\pm$ 0.005
    &\cellcolor[HTML]{ebf4e7} 0.63 $\pm$ 0.006&\cellcolor[HTML]{bcdaaf}1.17 $\pm$ 0.004&\cellcolor[HTML]{bad9ad}1.75 $\pm$ 0.003&-&-&-
    \\
    Soft-labels ($\beta$=30)* &58.01 $\pm$ 0.126 
    &2.31 $\pm$ 0.009 &0.97 $\pm$ 0.007 
    &\cellcolor[HTML]{ebf4e7}1.40 $\pm$ 0.005 &\cellcolor[HTML]{ddecd6}1.91 $\pm$ 0.005&-&-&-
    \\
    Soft-labels ($\beta$=4)* &22.66 $\pm$ 0.262 &\cellcolor[HTML]{d4e7cc}2.06 $\pm$ 0.012 &1.60 $\pm$ 0.007 &1.72 $\pm$ 0.008 
    &2.14 $\pm$ 0.007
    &-&-&-
    \\
    Flamingo-I &\cellcolor[HTML]{fbfdfa}70.37 $\pm$ 0.124&\cellcolor[HTML]{f8fbf7}2.14 $\pm$ 0.011&\cellcolor[HTML]{ebf4e7}0.63 $\pm$ 0.004&1.78 $\pm$ 0.008
    &3.27 $\pm$ 0.008
    &0.47 $\pm$ 0.002&0.34 $\pm$ 0.003&0.09 $\pm$ 0.002
    \\
    Flamingo-II &70.19 $\pm$ 0.168&2.16 $\pm$ 0.006&\cellcolor[HTML]{f8fbf7}0.64 $\pm$ 0.004&1.77 $\pm$ 0.007
    &3.30 $\pm$ 0.010
    &0.47 $\pm$ 0.002&0.35 $\pm$ 0.004&0.09 $\pm$ 0.002
    \\
    HAFeat-I &\cellcolor[HTML]{c6e0bb}70.96 $\pm$ 0.133&\cellcolor[HTML]{fcfdfc}2.15 $\pm$ 0.022&\cellcolor[HTML]{ddecd7}0.62 $\pm$ 0.006&1.55 $\pm$ 0.207&2.67 $\pm$ 0.406
    &0.55 $\pm$ 0.044&0.47 $\pm$ 0.091&0.23 $\pm$ 0.107\\
    HAFeat-II &70.26 $\pm$ 0.292&2.19 $\pm$ 0.012&0.65 $\pm$ 0.008&\cellcolor[HTML]{f7fbf6}1.46 $\pm$ 0.004&2.46 $\pm$ 0.005&0.55 $\pm$ 0.001&\cellcolor[HTML]{f4f9f2}0.51 $\pm$ 0.002&\cellcolor[HTML]{f4f9f2}0.28 $\pm$ 0.002\\
    HAFrame &\cellcolor[HTML]{b6d7a8}71.13 $\pm$ 0.139&\cellcolor[HTML]{d0e5c7}2.05 $\pm$ 0.007&\cellcolor[HTML]{b6d7a8}0.59 $\pm$ 0.001&\cellcolor[HTML]{b6d7a8}1.14 $\pm$ 0.002&\cellcolor[HTML]{b8d8aa}1.74 $\pm$ 0.001&\cellcolor[HTML]{b8d8aa}0.89 $\pm$ 0.001&\cellcolor[HTML]{b8d8aa}0.70 $\pm$ 0.001&\cellcolor[HTML]{b8d8aa}0.79 $\pm$ 0.001\\
    
    Hier-COS &\cellcolor[HTML]{b8d8aa}71.15 $\pm$ 0.037&\cellcolor[HTML]{cce3c2}2.06 $\pm$ 0.004&\cellcolor[HTML]{b6d7a8}0.59 $\pm$ 0.002&\cellcolor[HTML]{b6d7a8}1.13 $\pm$ 0.001&\cellcolor[HTML]{b6d7a8}	1.72 $\pm$ 0.001&\cellcolor[HTML]{b8d8aa}0.96 $\pm$ 0.000&\cellcolor[HTML]{b8d8aa}0.71 $\pm$ 0.000&\cellcolor[HTML]{b8d8aa}0.81 $\pm$ 0.001\\ \midrule
    ViT-Cross Entropy& 78.39 $\pm$ 0.068&\cellcolor[HTML]{b6d7a8}1.72 $\pm$ 0.005&0.37 $\pm$ 0.001&1.38 $\pm$ 0.001&2.61 $\pm$ 0.002& 0.53 $\pm$ 0.000&0.52 $\pm$ 0.001&0.25 $\pm$ 0.001\\ 
    ViT-Hier-COS& \cellcolor[HTML]{b6d7a8}80.81 $\pm$ 0.063&1.73 $\pm$ 0.011&\cellcolor[HTML]{b6d7a8}0.33 $\pm$ 0.003&\cellcolor[HTML]{b6d7a8}0.97 $\pm$ 0.002&\cellcolor[HTML]{b6d7a8}1.61 $\pm$ 0.001& \cellcolor[HTML]{b6d7a8} 0.98 $\pm$ 0.001&\cellcolor[HTML]{b6d7a8} 0.80 $\pm$ 0.001&\cellcolor[HTML]{b6d7a8} 0.86 $\pm$ 0.002\\ 
    \bottomrule
    \end{tabular}
}
\caption{Results comparing the performance of hierarchical classification on the test set of iNaturalist-19. The top partition compares the performance using ResNet-50 as the feature extractor, while the bottom partition compares the performance using a pre-trained ViT.}\label{tab:hfgvc-inat}
\end{table*}

\begin{table*}[!htp]\centering
\scriptsize
\resizebox{\textwidth}{!}{
    \begin{tabular}{l|ccc|ccc|ccc}
    \toprule
     \textbf{Method} &    \textbf{Accuracy $(\uparrow)$}& \textbf{FPA $(\uparrow)$}& \textbf{HOPS$(\uparrow)$}& \textbf{Accuracy $(\uparrow)$} &\textbf{FPA $(\uparrow)$}& \textbf{HOPS $(\uparrow)$} &  \textbf{Accuracy $(\uparrow)$}&\textbf{FPA $(\uparrow)$}&\textbf{HOPS $(\uparrow)$}\\\midrule
    Flamingo-I &   80.71 $\pm$ 0.408
    &77.92 $\pm$ 0.579&0.58 $\pm$ 0.004
    & \cellcolor[HTML]{b6d7a8}77.87 $\pm$ 0.137
    &75.19 $\pm$ 0.333
    &0.57 $\pm$ 0.002
     &  70.37 $\pm$ 0.124
    &68.63 $\pm$ 0.139  &0.47 $\pm$ 0.002
    \\
    Flamingo-II &   80.07 $\pm$ 0.364
    &77.70 $\pm$ 0.641&0.57 $\pm$ 0.005
    & 75.26 $\pm$ 0.508
    &72.43 $\pm$ 0.336
    &0.51 $\pm$ 0.003
     &  70.19 $\pm$ 0.168
    &68.75 $\pm$ 0.163  &0.47 $\pm$ 0.002
    \\
    HAFeat-I &   73.30 $\pm$ 0.361
    &54.38 $\pm$ 1.718&0.53 $\pm$ 0.002
    & 77.51 $\pm$ 0.224
    &\cellcolor[HTML]{e5f0e0}75.52 $\pm$ 0.223
    &\cellcolor[HTML]{f4f9f2}0.69 $\pm$ 0.001
     &  70.96 $\pm$ 0.133
    &33.36 $\pm$ 36.553  &0.55 $\pm$ 0.044
    \\
    HAFeat-II &   74.53 $\pm$ 0.592
    &55.73 $\pm$ 2.676&0.53 $\pm$ 0.009& 76.47 $\pm$ 0.219&73.24 $\pm$ 0.208
    &0.63 $\pm$ 0.001
     &  70.26 $\pm$ 0.292&68.01 $\pm$ 0.256  
    &0.55 $\pm$ 0.001
    \\
    Hier-COS &   \cellcolor[HTML]{b6d7a8}81.74 $\pm$ 1.000&\cellcolor[HTML]{b6d7a8}81.56 $\pm$ 1.020&\cellcolor[HTML]{b6d7a8}0.89 $\pm$ 0.007& \cellcolor[HTML]{b6d7a8}77.79 $\pm$ 0.145&\cellcolor[HTML]{b6d7a8}77.37 $\pm$ 0.118&\cellcolor[HTML]{b6d7a8}0.93 $\pm$ 0.001 &  \cellcolor[HTML]{b6d7a8}71.15 $\pm$ 0.037&\cellcolor[HTML]{b6d7a8}69.93 $\pm$ 0.028&
    \cellcolor[HTML]{b6d7a8}0.96 $\pm$ 0.000\\
    \bottomrule
    \end{tabular}
}
\caption{Hierarchical consistency on (left) FGVC-Aircraft, (center) CIFAR-100 and (right) iNat-19.} \label{tab:hier-consistency}
\end{table*}

%% file: sec/discussion.tex
\section{Conclusion}\label{sec:concl}

We introduced Hier-COS, a unified framework for hierarchical multi-class and multi-label classification, which ensures hierarchical consistency while implicitly adapting the learning capacity of each class based on hierarchy. Unlike prior methods that rely on multi-head architectures or manifold operations, it constructs hierarchy-aware representations through an orthogonal subspace composition that integrates seamlessly with existing deep backbones. 
Empirical evaluations show that Hier-COS consistently reduces the severity of mistakes, improves hierarchical consistency and achieves state-of-the-art hierarchical performance. The proposed HOPS metric addresses long-standing limitations in hierarchical evaluation by incorporating rank order and tree structure, offering a more interpretable and discriminative assessment of hierarchical models.

\section*{Acknowledgment}
This work was done under the ANRF (previously SERB) Grant\# CRG/2020/006049 from Govt. of India.  The authors acknowledge additional compute support from the Infosys Center for Artificial Intelligence at IIIT-Delhi. The authors also thank Anjaneya Sharma from IIIT-Delhi for his support in running additional analysis experiments in the post-submission phase of this paper.

%% file: sec/X_supp_cvpr.tex
We structure this supplementary material as follows, providing relevant references (in {\color{blue} blue}) to the main paper. 
In Sections \ref{sec:proof-theorem1} and \ref{sec:proof-proposition1}, we prove that our proposed formulation of Hier-COS results in an HAVS and is thereby hierarchically consistent (as mentioned in {\color{blue} Sec. 3.2} of the main paper). We then discuss a few important geometric properties of Hier-COS in Section \ref{sec:properties-hiercos}.
In Section \ref{supp:expt-setup}, we provide details about our experimental setup, dataset details, comparative baselines and implementation details (as mentioned in the {\color{blue} Sec. 5} of the main paper). In Section \ref{supp:expt}, we discuss the results on tieredImageNet-H (as mentioned in {\color{blue} Sec. 5} of the main paper). We also present the results of our ablation study and include qualitative results demonstrating the quality of feature representations learned by Hier-COS in comparison to previous methods. In Section \ref{supp:review}, we review the existing metrics used to evaluate hierarchical classification, information retrieval and ranking systems (as mentioned in {\color{blue} Sec. 4.1} of the main paper). We formally define HOPS in Section \ref{supp:hops}, and show its computation with an example (as mentioned in \textcolor{blue}{Sec. 4.2} of the main paper). We define and discuss the limitations of each metric. In Section \ref{supp:discussion}, we provide additional discussion on our framework, its complexity and some of the future research directions. We also provide additional supplementary files with our submission, which we summarize in Section \ref{supp:supp}. 

\input{sec/X_supp_cvpr_hiercos}
\input{sec/X_supp_cvpr_exp}
\input{sec/X_supp_cvpr_lit_review}
\input{sec/X_supp_cvpr_hops}

\section{Additional Discussions} \label{supp:discussion}

\noindent \textbf{Computational Overhead:}
In the worst case, the complexity for a complete $b$-ary tree is $\mathcal{O}(\frac{bK - 1}{b - 1}) \approx \mathcal{O}(K)$. However, in the real world, the complexity is much lower because the hierarchy trees are rarely complete. The values of ($K$, $n$) for all the datasets is given in table Table \ref{tab:dataset_statistics}. Moreover, unlike Flamingo and HAFeat, we do not append dummy nodes to have all the leaf nodes at the same level. Additionally, we do not train separate classifiers for each hierarchical level. Hence, we are resource-efficient when compared to these methods.

\noindent \textbf{Reducing Complexity when Severity of Mistakes is not important:}
The ablation experiments suggested that the level-wise accuracy for all the datasets is maintained even when the subspaces are defined using $V_i = \text{span}(\{e_i\})$, while the severity of mistakes is compromised. This is because all the leaf classes are orthogonal and have no hierarchical awareness. As mentioned earlier, the Hier-COS framework learns a single classifier that is capable of classifying at all levels, we can reduce the complexity of the network if the severity of mistakes is not conccerned. Specifically, for predicting classes at level $l$, we can prune the network by removing all the weights corresponding to $\{\mathcal{E} \setminus \mathcal{E}^{(l)}\}$.

\noindent \textbf{Extending to DAGs and Arbitrary Groups:} 
Previous methods \cite{flamingo_2021_Chang, baseline-hafeat} learn the semantic relationship between the leaf classes by learning a separate classifier for every level in the hierarchy tree using a multi-class classification loss with additional constraints. However, in case the relationship is not represented as a tree \cite{Ridnik_ImageNet21K_pretrain, peng2021rp2klargescaleretailproduct, bai2020products10klargescaleproductrecognition}, these methods can not be used and are non-trivial to extend. HAFrame \cite{baseline-Liang_2023-haframe} fixes the frames based on the LCA distance between leaf nodes. However, in the case of a Directed Acyclic Graph (DAG), there can be multiple lowest common ancestors that makes this formulation ambiguous and non-trivial to apply in such a setting. In contrast, we formulate Hier-COS based on the number of common ancestors between any two leaf nodes, which is unambiguous for DAGs or any arbitrary grouping that defines the hierarchical structure.

\noindent \textbf{Designing Kernel Tricks:}
Kernel methods have been widely used in various domains as they allow embeddings to be learned in a very high dimensional space while actually operating in a lower dimensional feature space. At the core, kernel methods use kernel functions that compute the inner product in a lower-dimensional feature space while implicitly transforming the feature embeddings into a higher-dimensional vector space where patterns become easier to recognize. We re-iterate that Hier-COS is a framework that defines a vector space $V_\mathcal{T}$ which implicitly captures the properties of the hierarchical structure. Therefore, we hypothesize that kernel tricks can be designed for implicitly transforming the feature vectors from $Z$ into $V_\mathcal{T}$ without operating in the $n$-dimensional vector space. This will especially be beneficial when the number of leaf nodes is very high and the given semantic relationship is complex.

\section{Additional Supplementary} \label{supp:supp}
This section summarizes all the supplementary files shared with the main paper. In \textit{\texttt{Results and Analysis.pdf}}, we share pre-computed results of our analysis on CIFAR-100. It includes visualization of hierarchy trees for all the datasets (demonstrating the complexity, imbalance, etc.), error \& metric analysis, visualization of the weights $\mathbf{w}$ used for HOPS and HOPS@$k$ (depicting the impact of proposed weight decay function), and results of FPA computed based on the approach presented by HAFeat. We will release the code upon acceptance.

%% file: sec/X_supp_cvpr_hiercos.tex
\section{Proof of Theorem 1} \label{sec:proof-theorem1}

\begin{theorem} \label{theorem:hiercos}
    Consider a vector space $V_\mathcal{T}$ and its subspaces for fine-grained classes $\{V_{y_1}, \cdots, V_{y_K}\}$ spanned by $\{\mathcal{E}_{y_1}, \cdots, \mathcal{E}_{y_K}\}$ defined using the hierarchy tree $\mathcal{T}$, as discussed above. For all $\mathbf{x} \in V_\mathcal{H}$, if $\mathbf{x} \in V_{y_i}$ and $\langle \mathbf{x}, e_{i} \rangle^2 > 0, ~\forall~  e_{i} \in \mathcal{E}_{y_i}$, then $V_\mathcal{H}$ is an $n$-dimensional HAVS induced by an LCA-based tree distance function $D_\mathcal{T}$.
\end{theorem}

\begin{proof}
    To prove $V_\tree$ is an HAVS, we must satisfy \textcolor{blue}{Def. 1}. Let $\vect{x}$ be a feature vector for class $y_i$ such that the theorem's conditions hold.
    We must show that for any two other classes $y_j, y_k \in \set{Y}$:
    \begin{align*}
        \text{if} \quad \distT{y_i}{y_j} &< \distT{y_i}{y_k} \quad \\
        \text{then} \quad |D_i - D_j| &< |D_i - D_k|
    \end{align*}
    where, $D_\mathcal{T}$ is an LCA-based distance function, which implies that $D_\mathcal{T}(y_i, y_j) < D_\mathcal{T}(y_i, y_k)$ if and only if \quad $\vert f_a(v_i) \cap f_a(v_j) \vert > \vert f_a(v_i) \cap f_a(v_k) \vert$.

    First, by the theorem's assumption $\vect{x} \in \subspace{y_i}$, the projection of $\vect{x}$ onto its own subspace is $\vect{x}$ itself. Thus, its distance to its own subspace is zero:
    \begin{equation} \label{eq:proof1-d-i}
        D_i = \norm{\vect{x} - \proj{y_i}(\vect{x})} = 0
    \end{equation}
    The HAVS condition simplifies to proving:
    \begin{align*}
        \text{if} \quad \distT{y_i}{y_j} &< \distT{y_i}{y_k} \\
        \text{then} \qquad D_j &< D_k
    \end{align*}

    From the main paper's definition (\textcolor{blue}{Eq. 2}), the squared distance from $\vect{x}$ to any subspace $V_{y_j}$ is the sum of squared components of $\vect{x}$ on the orthogonal complement basis $\neg\basis{y_j}$:
    \begin{equation} \label{eq:proof1-d-j-def}
        D_j^2(\vect{x}) = \sum_{e_m \in \neg\basis{y_j}} \ip{\vect{x}}{e_m}^2
    \end{equation}

    Now, we use the first assumption, $\vect{x} \in \subspace{y_i}$. This implies that $\vect{x}$ has no components outside of $\subspace{y_i}$. Formally, $\ip{\vect{x}}{e_m} = 0, \forall e_m \in \neg\basis{y_i}$, i.e., $\mathbf{x}$ is orthogonal to all the basis vectors not in $\basis{y_i}$. We can substitute this into Eq. \ref{eq:proof1-d-j-def}. The only basis vectors $e_m \in \neg\basis{y_j}$ that can contribute to the sum are those that are also in $\basis{y_i}$.
    \begin{equation*}
        D_j^2(\vect{x}) = \sum_{e_m \in \neg\basis{y_j} \cap \basis{y_i}} \ip{\vect{x}}{e_m}^2
    \end{equation*}
    This set is equivalent to the set difference $\basis{y_i} \setminus \basis{y_j}$. Thus, the distance is the energy of $\vect{x}$ projected onto the basis vectors that are in $y_i$'s path but not in $y_j$'s path:
    \begin{equation} \label{eq:proof1-d-j-final}
        D_j^2(\vect{x}) = \sum_{e_m \in \basis{y_i} \setminus \basis{y_j}} \ip{\vect{x}}{e_m}^2
    \end{equation}
    Similarly, for class $y_k$:
    \begin{equation} \label{eq:proof1-d-k-final}
        D_k^2(\vect{x}) = \sum_{e_m \in \basis{y_i} \setminus \basis{y_k}} \ip{\vect{x}}{e_m}^2
    \end{equation}

    Since we know that  $\vert f_a(v_i) \cap f_a(v_j) \vert > \vert f_a(v_i) \cap f_a(v_k) \vert$, i.e., the set of shared ancestors for $y_i$ and $y_j$ is larger than for $y_i$ and $y_k$, therefore:
    \begin{equation*}
        |\basis{y_i} \cap \basis{y_j}| > |\basis{y_i} \cap \basis{y_k}|
    \end{equation*}
    Since $\basis{y_i} \setminus \basis{y_j}$ is equivalent to $\basis{y_i} \setminus (\basis{y_i} \cap \basis{y_j})$, a larger intersection implies a smaller set difference. Therefore:
    \begin{equation*}
        |\basis{y_i} \setminus \basis{y_j}| < |\basis{y_i} \setminus \basis{y_k}|
    \end{equation*}
    Furthermore, because the LCA path is shared, the set difference $\basis{y_i} \setminus \basis{y_j}$ is a strict subset of $\basis{y_i} \setminus \basis{y_k}$:
    \begin{equation} \label{eq:proof1-subset}
        (\basis{y_i} \setminus \basis{y_j}) \subset (\basis{y_i} \setminus \basis{y_k})
    \end{equation}
    
    Finally, we use the second assumption of the theorem: $\ip{\vect{x}}{e}^2 > 0$ for all $e \in \basis{y_i}$.
    Since the sum for $D_j^2(\vect{x})$ (Eq. \ref{eq:proof1-d-j-final}) is over a strict subset of the basis vectors used for $D_k^2(\vect{x})$ (Eq. \ref{eq:proof1-d-k-final}), and all terms are strictly positive, the sum over the smaller set must be smaller.
    \begin{equation*}
        D_j^2(\vect{x}) < D_k^2(\vect{x}) \implies D_j(\vect{x}) < D_k(\vect{x})
    \end{equation*}
    Combining this with $D_i(\vect{x})=0$ (Eq. \ref{eq:proof1-d-i}), we have:
    \begin{equation*}
        |D_i(\vect{x}) - D_j(\vect{x})| < |D_i(\vect{x}) - D_k(\vect{x})|
    \end{equation*}
    This satisfies \textcolor{blue}{Def. 1}, proving that $V_\tree$ is an HAVS.
\end{proof}

\section{Proof of Proposition 1} \label{sec:proof-proposition1}

    \begin{lemma} \label{lemma:parent-child-dist}
            For any given node $v_j \in \nodes_\tree$ and its parent $v_p = \parent{v_j}$, $\subspace{j} \subset \subspace{p}$.
    \end{lemma}
    \begin{proof} 
        By definition, 
        \begin{align*}
            \basis{j} = \basis{j}^a \cup \{e_j\} \cup \basis{j}^d \\ 
            \basis{p} = \basis{p}^a \cup \{e_p\} \cup \basis{p}^d
        \end{align*}
        The ancestors of $v_j$ are $v_p$ and all of $v_p$'s ancestors, i.e.
        \begin{align*}
            \basis{j}^a = \basis{p}^a \cup \{e_p\}
        \end{align*}
        The descendants of $v_p$ include $v_j$ and all of $v_j$'s descendants, i.e. 
        \begin{align*}
            \basis{p}^d \supset \{e_j\} \cup \basis{j}^d
        \end{align*}
        Substituting these into the definition of $\basis{p}$
        \begin{align*}
            \basis{p} &= \basis{p}^a \cup \{e_p\} \cup \basis{p}^d \\
            &\supset \basis{j}^a \cup (\{e_j\} \cup \basis{j}^d) \\
            &\supset \basis{j}
        \end{align*}
        Thus, $\basis{j} \subset \basis{p}$, which implies $\subspace{j} \subset \subspace{p}$.
    \end{proof}

\begin{proposition} \label{proposition:consistency}
    The label hierarchy predicted using Hier-COS is consistent across all the levels, i.e.,  $\{\hat{y}^{(1)}, \ldots, \hat{y}^{(l)}, \ldots, \hat{y}^{(H)}\}$ is a valid path in the tree.
\end{proposition}

\begin{proof}
    Hierarchical consistency requires that for any predicted node $\hat{y}^{(l)}$ at level $l > 1$, its predicted parent $\hat{y}^{(l-1)}$ must be its true parent, i.e., $\hat{y}^{(l-1)} = \parent{\hat{y}^{(l)}}$.
    
    The prediction at any level $l$ is the node $v_j \in \nodes_\tree^{(l)}$ that minimizes the distance $D_j^2(\vect{x})$:
    \begin{equation} \label{eq:prop1-pred}
        \hat{y}^{(l)} = \arg \min_{v_j \in \nodes_\tree^{(l)}} D_j^2(\vect{x})
    \end{equation}
    
    We prove consistency by showing that for an ideal feature vector $\vect{x}$ corresponding to a single leaf class $\hat{y}_{leaf}$, the distance to any true ancestor is 0, while the distance to any non-ancestor is strictly positive.
    
    Let $\hat{y}_{leaf}$ be the predicted leaf node. In an ideal setting, $\vect{x} \in \subspace{\hat{y}_{leaf}}$. This implies $\ip{\vect{x}}{e_k} = 0, \forall e_k \in \neg\basis{\hat{y}_{leaf}}$. Let $\hat{y}^{(l)}$ be any true ancestor of $\hat{y}_{leaf}$ at level $l$.
    By recursive application of Lemma \ref{lemma:parent-child-dist}, we have $\subspace{\hat{y}_{leaf}} \subset \subspace{\hat{y}^{(l)}}$.
    Since $\vect{x} \in \subspace{\hat{y}_{leaf}}$, it follows that $\vect{x} \in \subspace{\hat{y}^{(l)}}$.
    Therefore, the distance from $\vect{x}$ to any of its true ancestors is zero:
    \begin{equation} \label{eq:prop1-ancestor-dist}
        D_{\hat{y}^{(l)}}^2(\vect{x}) = \norm{\vect{x} - \proj{\hat{y}^{(l)}}(\vect{x})} = 0
    \end{equation}
    
    Now, let $v_q$ be any \textit{other} node at level $l$ (i.e., $v_q \neq \hat{y}^{(l)}$).
    Since $v_q$ is at the same level as $\hat{y}^{(l)}$ but is not the same node, $v_q$ cannot be an ancestor of $\hat{y}_{leaf}$.
    By the definition of the subspace $\subspace{q}$, its basis $\basis{q}$ consists of $\basis{q}^a \cup \{e_q\} \cup \basis{q}^d$.
    Since $v_q$ is not an ancestor of $\hat{y}_{leaf}$, $e_{\hat{y}_{leaf}} \notin \basis{q}^a \cup \{e_q\}$.
    Since $\hat{y}_{leaf}$ is not a descendant of $v_q$, $e_{\hat{y}_{leaf}} \notin \basis{q}^d$.
    Therefore, the basis vector $e_{\hat{y}_{leaf}}$ is not in $\basis{q}$:
    \begin{equation*}
        e_{\hat{y}_{leaf}} \in \neg\basis{q}
    \end{equation*}
    
    Now consider the distance to this ``incorrect'' node $v_q$:
    \begin{equation*}
        D_q^2(\vect{x}) = \sum_{e_k \in \neg\basis{q}} \ip{\vect{x}}{e_k}^2
    \end{equation*}
    This sum includes the term $\ip{\vect{x}}{e_{\hat{y}_{leaf}}}^2$.
    The condition in Theorem \ref{theorem:hiercos} is designed to ensure the feature vector $\vect{x}$ has non-zero energy on its leaf node, i.e., $\ip{\vect{x}}{e_{\hat{y}_{leaf}}}^2 > 0$.
    
    Since $D_q^2(\vect{x})$ is a sum of non-negative terms (squared inner products) and includes at least one strictly positive term ($\ip{\vect{x}}{e_{\hat{y}_{leaf}}}^2$), we have:
    \begin{equation} \label{eq:prop1-other-dist}
        D_q^2(\vect{x}) \ge \ip{\vect{x}}{e_{\hat{y}_{leaf}}}^2 > 0
    \end{equation}
    
    Comparing Eq. \ref{eq:prop1-ancestor-dist} and \ref{eq:prop1-other-dist}, at any level $l$, the distance to the true ancestor $\hat{y}^{(l)}$ is 0, while the distance to any other node $v_q$ is strictly positive.
    
    Therefore, the prediction at level $l$ must be the true ancestor:
    \begin{equation*}
        \hat{y}^{(l)} = \arg \min_{v_j \in \nodes_\tree^{(l)}} D_j^2(\vect{x}) = \hat{y}^{(l)}
    \end{equation*}
    This holds for all levels, and thus $\hat{y}^{(l-1)} = \parent{\hat{y}^{(l)}}$ for all $l$. The predicted path is guaranteed to be consistent.

\end{proof}

\section{Properties of Hier-COS} \label{sec:properties-hiercos}
\begin{enumerate}
    \item \textbf{Isometry between the subspace distance and LCA-based tree distance:} The sorted order of pairwise distances between the subspaces $V_{y_i}$ and $V_{y_j} ~\forall j \in [K]$, as defined in Eq. \ref{eq:proof1-d-j-def}, results in an order which is exactly the same as the partial preference order obtained using the LCA-based tree distance metric. Intuitively, with the increase in the number of common ancestors between any two leaf nodes, the overlap of the subspaces also increases, thereby reducing their distance. Therefore, feature vectors in Hier-COS are inherently consistent with the semantic similarity obtained from the tree.
    \item \textbf{Isometry between Grassmannian distance and LCA-based tree distance:} For any two distinct leaf nodes $y_i$ and $y_j$, the squared Grassmannian distance between their subspaces is directly proportional to the LCA-based tree distance between them in the hierarchy tree $\mathcal{T}$. Intuitively, with orthogonal basis vectors as in Hier-COS, the principal angles between any two subspaces can either be exactly zero (common ancestors) or $\pi/2$ (distinct nodes). Since the Grassmannian distance between two subspaces is directly proportional to the sum of their principal angles, therefore, it is directly proportional to the LCA-based tree distance.
    \item \textbf{Hierarchy-Adaptive Capacity:} The dimensionality of the subspace $V_y$ adapts to the semantic complexity of the class $y$. Superclasses have higher dimensionality to capture diverse features, while fine-grained classes have constrained dimensionality for specificity. $\basis{d}$ for a superclass provides the necessary capacity to represent the union of features from all diverse sub-categories (e.g., $V_{Animal}$ must accommodate features of both birds and fish). Whereas $\basis{a}$ acts as a constraint, forcing the model to learn a specific representation for leaf class that is a refinement of its ancestors, without the excess degrees of freedom found in superclasses.
    \item \textbf{Unified Classifier for Hierarchy-Aware and Hierarchical Multi-Label Classification:} The distance of the feature vector $\mathbf{x} \in V_\mathcal{T}$ can be computed from any subspace $V_i \subset V_\mathcal{T}$. This property allows computing the pairwise distance of $\mathbf{x}$ from all the subspaces corresponding to each node at a particular height $h$, therefore, encouraging predicting a class label for each hierarchical level $h \in [H]$. This enables hierarchical multi-label classification.
\end{enumerate}

%% file: sec/X_supp_cvpr_exp.tex
\section{Experimental Setup} \label{supp:expt-setup}
\textbf{Dataset:} We evaluate our approach on the test set of four datasets: FGVC-Aircraft \cite{fgvc-aircraft}, CIFAR-100 \cite{krizhevsky_2009_CIFAR100}, iNaturalist-19 \cite{van_inat_2018_CVPR} and tieredImageNet-H \cite{ren_tiered_18_ICLR}. For FGVC-Aircraft, we adopt the original hierarchy provided by the dataset. For CIFAR-100, we use the hierarchy tree provided by \cite{landrieu_2021_BMVC}, while for iNaturalist-19 and tieredImageNet-H, we use the hierarchies provided by \cite{bertinetto_2020_CVPR}. The statistics of these datasets are summarized in Table \ref{tab:dataset_statistics}.
\begin{table}[h]
\scriptsize
    \begin{center}
    \begin{tabular}{l|c|c|c|c|c|c}
        \toprule
        Dataset & $H$ & $K$ & $n$ & Train & Val & Test \\ \midrule
        FGVC-Aircraft &  3& 100 & 200
&3.3K &3.3K & 3.3K \\ 
        CIFAR-100 &  5& 100 & 134
&45K & 5K & 10K \\ 
        iNaturalist-19 &  7& 1010 & 1189
&187K & 40K & 40K \\ 
        tieredImageNet-H&  12& 608 & 842
&425K & 15K & 15K \\ \bottomrule
    \end{tabular}
    \end{center} 
    \caption{Statistics of the datasets while excluding the root node. $H$ is the height of the hierarchy tree,  $K$ is the number of fine-grained classes and $n$ is the total number of nodes in the tree. The right-most three columns represent the approximate number of samples in train, val and test sets, respectively.} 
    \label{tab:dataset_statistics}	
\end{table}

\noindent \textbf{Baseline Methods:} 
We compare our method with \cite{barz_2019_WACV, redmon_2017_yolo_CVPR, bertinetto_2020_CVPR, flamingo_2021_Chang, baseline-hafeat, baseline-Liang_2023-haframe} and a \textit{flat} cross-entropy-based classification approach. Following \cite{baseline-Liang_2023-haframe}, we also include baselines without the transformation layer (I) and with the transformation layer (II). For a fair comparison with the best competing methods, we re-did the experiments using the official codebase provided by the authors. For methods denoted with an asterisk (*), we simply use the metrics reported by \cite{baseline-hafeat}. Among all the baseline methods, \cite{flamingo_2021_Chang, baseline-hafeat} train an additional classifier for each level of granularity. We compare our results with them to evaluate hierarchical consistency. Although cross entropy and \cite{baseline-Liang_2023-haframe} do not explicitly allow classification across all the hierarchical levels, we compute the probabilities for a super-class as the sum of probabilities of all its children to compare the FPA on CIFAR-100 \cite{baseline-hafeat}. None of the previous methods evaluate the hierarchical performance of a transformer-based method. We create a simple baseline where we only fine-tune the output layer of a pre-trained ViT using the cross-entropy loss.

\noindent \textbf{Backbone Architectures:} Following the same strategy as \cite{baseline-Liang_2023-haframe, baseline-hafeat}, we use WideResNet architecture for CIFAR-100 and ResNet-50 for FGVC-Aircraft, iNaturalist-19 and tieredImageNet-H. In addition to the CNN-based feature extractors, we compare our method using a transformer-based backbone. We use a pre-trained ViT-MAE for iNaturalist-19
and a ViT-B16 for tiered-ImageNet-H datasets 
as the feature extractor and freeze the weights, i.e., we only learn a transformation map from the learned embedding space to Hier-COS.

\noindent \textbf{Evaluation Metrics:} We follow the same evaluation strategy as \cite{bertinetto_2020_CVPR, karthik_2021_ICLR, baseline-hafeat, baseline-Liang_2023-haframe} and report the top-1 accuracy, MS and AHD@\{1, 5, 20\} for all the baseline methods. Similar to \cite{baseline-Liang_2023-haframe}, we report the mean and 95\% confidence interval derived from the t-distribution with four degrees of freedom for five different seeds. Additionally, we report the proposed HOPS, HOPS@\{5, 20\} for the competing methods. To quantify hierarchical consistency, we also report the full-path accuracy (FPA) \cite{park2024learninghierarchicalsemanticclassification} for FGVC-Aircraft, CIFAR-100 and iNaturalist-19. We do not report FPA for tieredImageNet-H because the leaf classes in the corresponding hierarchy tree are not at the same level. Unlike previous methods, we do not artificially extend the leaf nodes to be at the same level, therefore, classes at levels $l \in [H]$ are not well defined. 

\begin{table*}[!htp]\centering
\scriptsize
\begin{tabular}{l|c|c|c|c|c|c|c|c}\toprule
\textbf{Method} & \textbf{Accuracy $(\uparrow)$} & \textbf{MS $(\downarrow)$} & \textbf{AHD@1 $(\downarrow)$} & \textbf{AHD@5 $(\downarrow)$} & \textbf{AHD@20 $(\downarrow)$} & \textbf{HOPS $(\uparrow)$} &\textbf{HOPS@5 $(\uparrow)$} &\textbf{HOPS@20 $(\uparrow)$} \\\midrule
Cross Entropy &\cellcolor[HTML]{b8d8ab}73.63 $\pm$ 0.312&\cellcolor[HTML]{fcfdfc}6.95 $\pm$ 0.021&\cellcolor[HTML]{b9d8ab}1.83 $\pm$ 0.023&5.66 $\pm$ 0.008
&7.29 $\pm$ 0.010
&0.58 $\pm$ 0.001&0.23 $\pm$ 0.003&0.14 $\pm$ 0.002
\\
Barz \& Denzler* &60.27 $\pm$ 0.240 
&\cellcolor[HTML]{daebd3}6.80 $\pm$ 0.019 &2.70 $\pm$ 0.022 
&5.48 $\pm$ 0.271 &\cellcolor[HTML]{cee4c5} 6.21 $\pm$ 0.005&-&
-&
-
\\
YOLO-v2* &66.02 $\pm$ 0.099 &6.99 $\pm$ 0.011 &2.38 $\pm$ 0.012 &\cellcolor[HTML]{dbebd4}5.05 $\pm$ 0.001 &\cellcolor[HTML]{f5faf4}6.17 $\pm$ 0.001&-&

-&
-
\\
HXE+CRM* &\cellcolor[HTML]{bbdaad}73.54 $\pm$ 0.150 &\cellcolor[HTML]{e0eeda}6.89 $\pm$ 0.027&\cellcolor[HTML]{b6d7a8}1.82 $\pm$ 0.016&\cellcolor[HTML]{b6d7a8}4.82 $\pm$ 0.006&\cellcolor[HTML]{b6d7a8}6.03 $\pm$ 0.004&-&
-&
-
\\
Soft-labels ($\beta$=30)* &69.31 $\pm$ 0.125 
&6.99 $\pm$ 0.007 
&2.15 $\pm$ 0.008 
&\cellcolor[HTML]{cbe2c1}4.95 $\pm$ 0.001 &\cellcolor[HTML]{daebd3}6.11 $\pm$ 0.001&-&

-&
-
\\
Soft-labels ($\beta$=4)* &17.28 $\pm$ 0.079 &7.54 $\pm$ 0.001 
&6.24 $\pm$ 0.005 &6.94 $\pm$ 0.005 
&7.25 $\pm$ 0.002
&-&
-&
-
\\
Flamingo-I &\cellcolor[HTML]{daebd3}72.27 $\pm$ 0.209&6.96 $\pm$ 0.016
&\cellcolor[HTML]{daead3}1.93 $\pm$ 0.017&5.77 $\pm$ 0.010
&7.42 $\pm$ 0.012
&0.58 $\pm$ 0.001&

0.22 $\pm$ 0.001&0.13 $\pm$ 0.002
\\
Flamingo-II &65.72 $\pm$ 1.551
&7.07 $\pm$ 0.036
&2.42 $\pm$ 0.122
&5.79 $\pm$ 0.023
&7.33 $\pm$ 0.018
&0.58 $\pm$ 0.001&0.22 $\pm$ 0.004&0.14 $\pm$ 0.002
\\
HAFeat-I &\cellcolor[HTML]{b6d7a8}73.49 $\pm$ 0.218
&6.92 $\pm$ 0.027
&\cellcolor[HTML]{b6d7a8}1.83 $\pm$ 0.016
&5.55 $\pm$ 0.020
&6.98 $\pm$ 0.017
&  0.65 $\pm$ 0.003&

 0.28 $\pm$ 0.002& 0.21 $\pm$ 0.003
\\
HAFeat-II &67.89 $\pm$ 1.772&7.05 $\pm$ 0.016&2.26 $\pm$ 0.128&5.62 $\pm$ 0.036&6.97 $\pm$ 0.015&  0.65 $\pm$ 0.004&0.26 $\pm$ 0.011&0.20 $\pm$ 0.007\\
HAFrame &\cellcolor[HTML]{b6d7a8}73.70 $\pm$ 0.284&\cellcolor[HTML]{e5f0e0}6.90 $\pm$ 0.007&\cellcolor[HTML]{b6d7a8}1.82 $\pm$ 0.019&\cellcolor[HTML]{cce3c3}4.96 $\pm$ 0.013&\cellcolor[HTML]{f1f7ee}6.16 $\pm$ 0.009&\cellcolor[HTML]{bfdbb2}0.87 $\pm$ 0.003&

\cellcolor[HTML]{bfdbb2}0.56 $\pm$ 0.008&\cellcolor[HTML]{bfdbb2}0.60 $\pm$ 0.009\\
Hier-CoS &\cellcolor[HTML]{daebd3}72.22 $\pm$ 0.211&\cellcolor[HTML]{b6d7a8}6.70 $\pm$ 0.009&\cellcolor[HTML]{b6d7a8}1.86 $\pm$ 0.012&\cellcolor[HTML]{b7d7a9}4.81 $\pm$ 0.006&\cellcolor[HTML]{bfdbb2}6.02 $\pm$ 0.002&\cellcolor[HTML]{bfdbb2}0.94 $\pm$ 0.001&\cellcolor[HTML]{bfdbb2}0.71 $\pm$ 0.003&\cellcolor[HTML]{bfdbb2}0.76 $\pm$ 0.002\\ 
\midrule
ViT-Cross Entropy&\cellcolor[HTML]{b6d7a8}75.10 $\pm$ 0.079 &6.77 $\pm$ 0.015 &1.69 $\pm$ 0.005 &5.44 $\pm$ 0.002 &7.00 $\pm$ 0.003 &  0.62 $\pm$ 0.000&0.31 $\pm$ 0.001&0.21 $\pm$ 0.001\\
ViT-Hier-CoS &74.71 $\pm$ 0.082&\cellcolor[HTML]{b6d7a8} 6.60 $\pm$ 0.023 &\cellcolor[HTML]{b6d7a8} 1.67 $\pm$ 0.009 &\cellcolor[HTML]{b6d7a8} 4.82 $\pm$ 0.007 &\cellcolor[HTML]{b6d7a8} 6.07 $\pm$ 0.003 &  \cellcolor[HTML]{bfdbb2}0.92 $\pm$ 0.001&\cellcolor[HTML]{bfdbb2}0.65 $\pm$ 0.004&\cellcolor[HTML]{bfdbb2}0.70 $\pm$ 0.004\\
\bottomrule
\end{tabular}
\caption{Results comparing the performance of hierarchical classification on the test set of tieredImageNet-H.}\label{tab:hfgvc-tiered}
\end{table*}

\subsection{Implementation Details}
Inspired by HAFrame \cite{baseline-Liang_2023-haframe}, the transformation module comprises of 5 linear layers, each with hidden units same as the number of nodes $n$, along with batch-norm layer and PReLU activation feeding to a fully connected layer with fixed weights $\{e_1 \ldots, e_n\}$ (Fig. \ref{fig:f-theta}). A small distinction from \cite{baseline-Liang_2023-haframe} is that we define the transformation module $f_\theta$ as the Transformation + Linear module of HAFrame. We re-emphasize that the weight vectors (equivalently, the basis set $\mathcal{E}$) can be arbitrary but are always orthonormal and fixed for the fifth linear layer.

\begin{figure}[!ht]
    \centering
    \includegraphics[trim={3cm 0 3cm 0},clip,width=\linewidth]{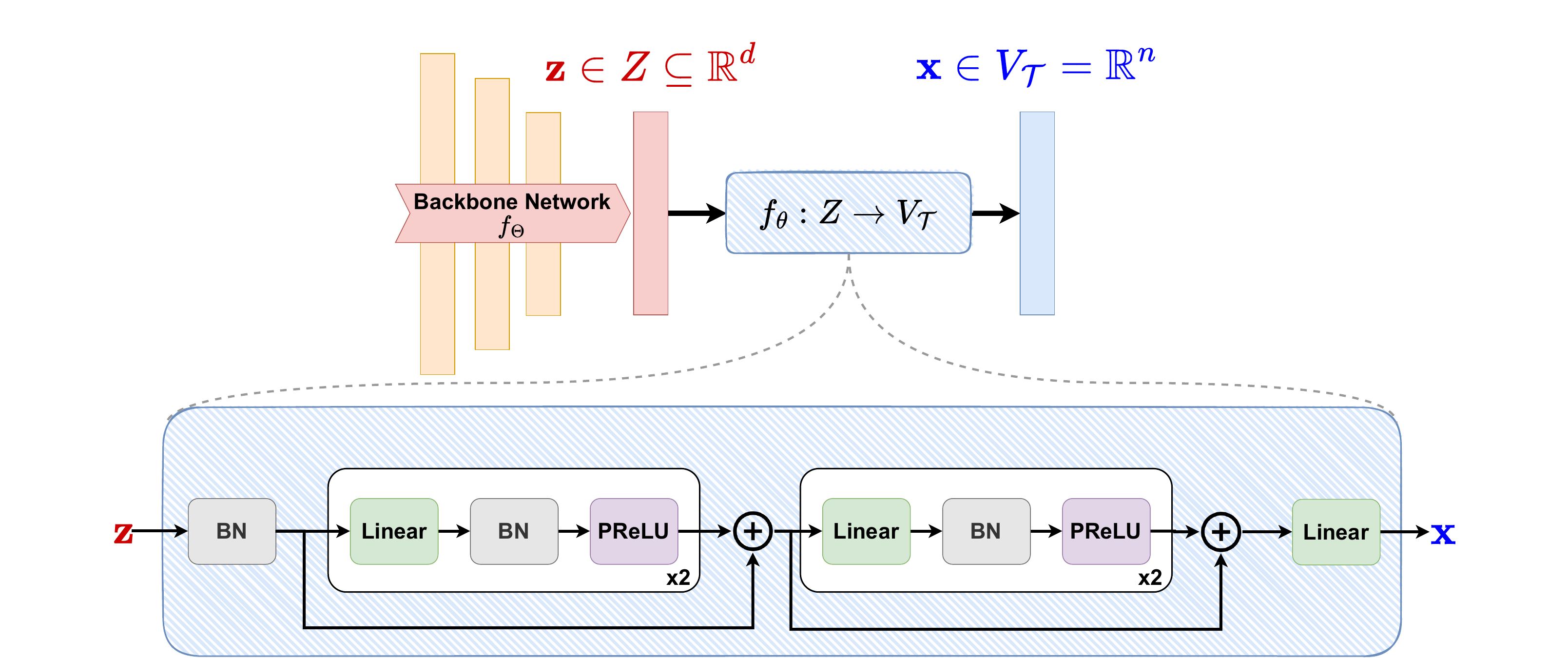}
    \caption{Illustration of the transformation module $f_\theta:Z\rightarrow V_\calT$. The components of $f_\theta$ are inspired from HAFrame \cite{baseline-Liang_2023-haframe}}
    \label{fig:f-theta}
\end{figure}

\subsection{Training Configuration}
We use the same training configurations as \cite{baseline-Liang_2023-haframe}. We obtained the best results on FGVC-Aircraft, CIFAR-100, iNaturalist-19 and tieredImageNet-H using $\alpha$ as $(0.1, 0.05, 0.001, 0.0001)$, respectively. To compare our method using a transformer-based backbone, we freeze the weights of the backbone and only fine-tune $f_\theta$ for $10$ epochs with $\alpha=1e-4$.

\section{Additional Experiments and Results} \label{supp:expt}

In this section, we include the results on tieredImageNet-H in Table \ref{tab:hfgvc-tiered}. Here, we observe a reduction of $1.48\%$ in the top-1 accuracy as compared to SOTA, while we observe significant improvements in MS, AHD@$\{5, 20\}$, HOPS and HOPS@$\{5,20\}$. AHD@1 is slightly reduced because of the degraded top-1 performance. We emphasize that tieredImageNet has a complex hierarchical structure (visualized in \textit{\texttt{Results and Analysis.pdf}}) with leaf nodes at different levels of the hierarchy. This makes learning algorithms difficult to optimize. This is also evident in SOTA. We observe that while HAFrame improves the top-1 performance, its overall performance is not better than the baseline HXE+CRM. The AHD@$k$ metrics for HXE+CRM are significantly better than HAFrame, while we show that with a slight reduction in top-1 performance, Hier-COS is able to achieve better performance across all the hierarchical metrics.

{\color{blue} Fig. 1} from the main paper, depicts the improvement in hierarchical performance. We observe a large drop in the percentage of samples with correct prediction orders using all the previous methods for all $k>1$ in {\color{blue} Fig. 1}. Although the drop in accuracy is not justified because, unlike \cite{bertinetto_2020_CVPR}, we demonstrated via other experiments that there is no trade-off between top-1 and hierarchical performance, we speculate that this occurs due to the design of a simple loss function that does not account for the imbalance in the height of subtrees. Specifically, each class in tieredImageNet-H spans a different number of dimensions, while all the classes in other datasets span exactly $H$ dimensions. We hypothesize that accounting for the dimensionality of subspaces should also improve the top-1 accuracy.

\subsection{Ablative Study} \label{sec:ablation}
\begin{figure*}
    \centering
    \begin{subfigure}[b]{0.33\textwidth}
        \centering
        \includegraphics[width=\linewidth]{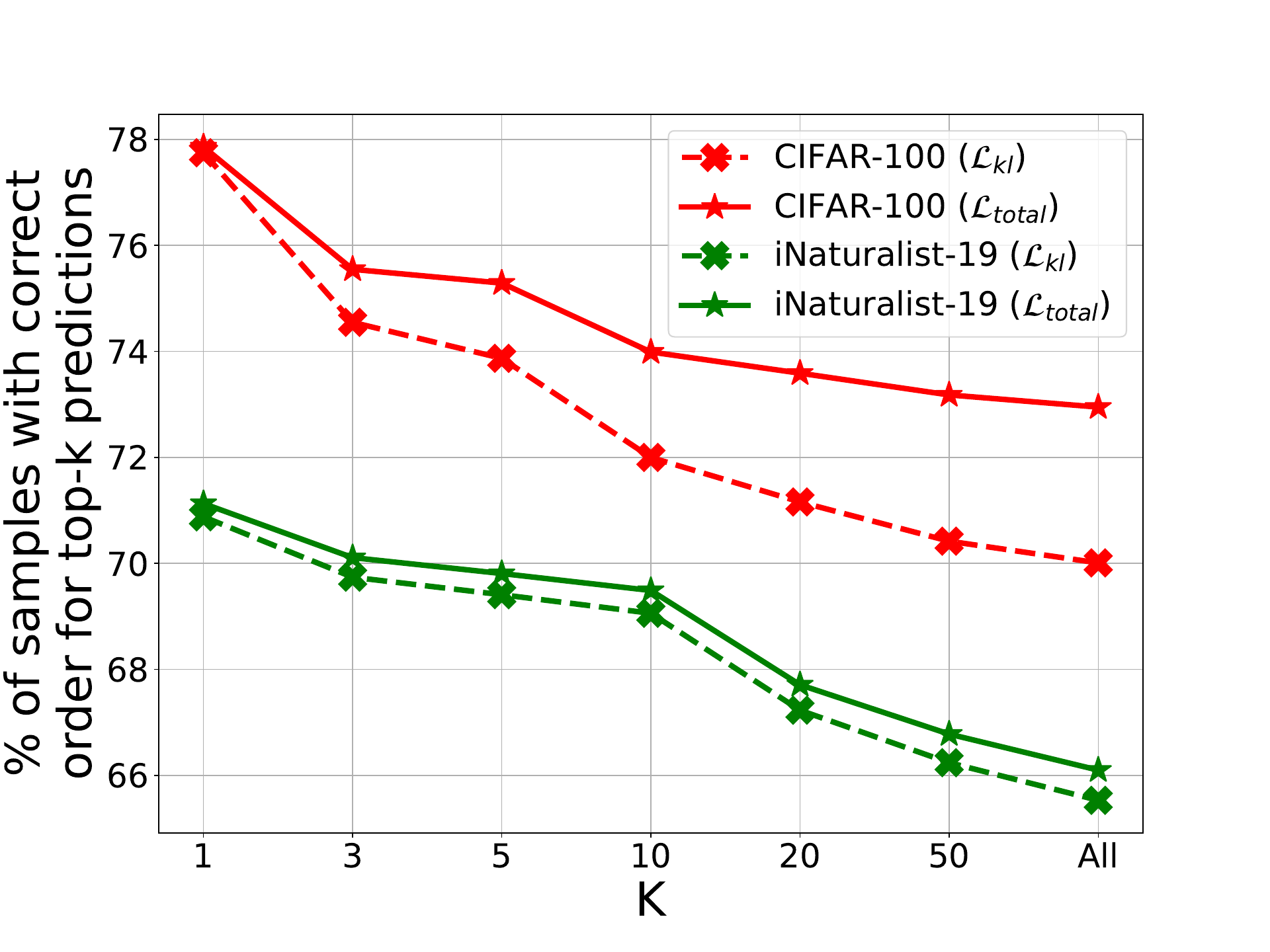}
        \caption{Impact of $\mathcal{L}_{reg}$} \label{fig:ablation-a}
    \end{subfigure}
    \begin{subfigure}[b]{0.33\textwidth}
        \centering
        \includegraphics[width=\linewidth]{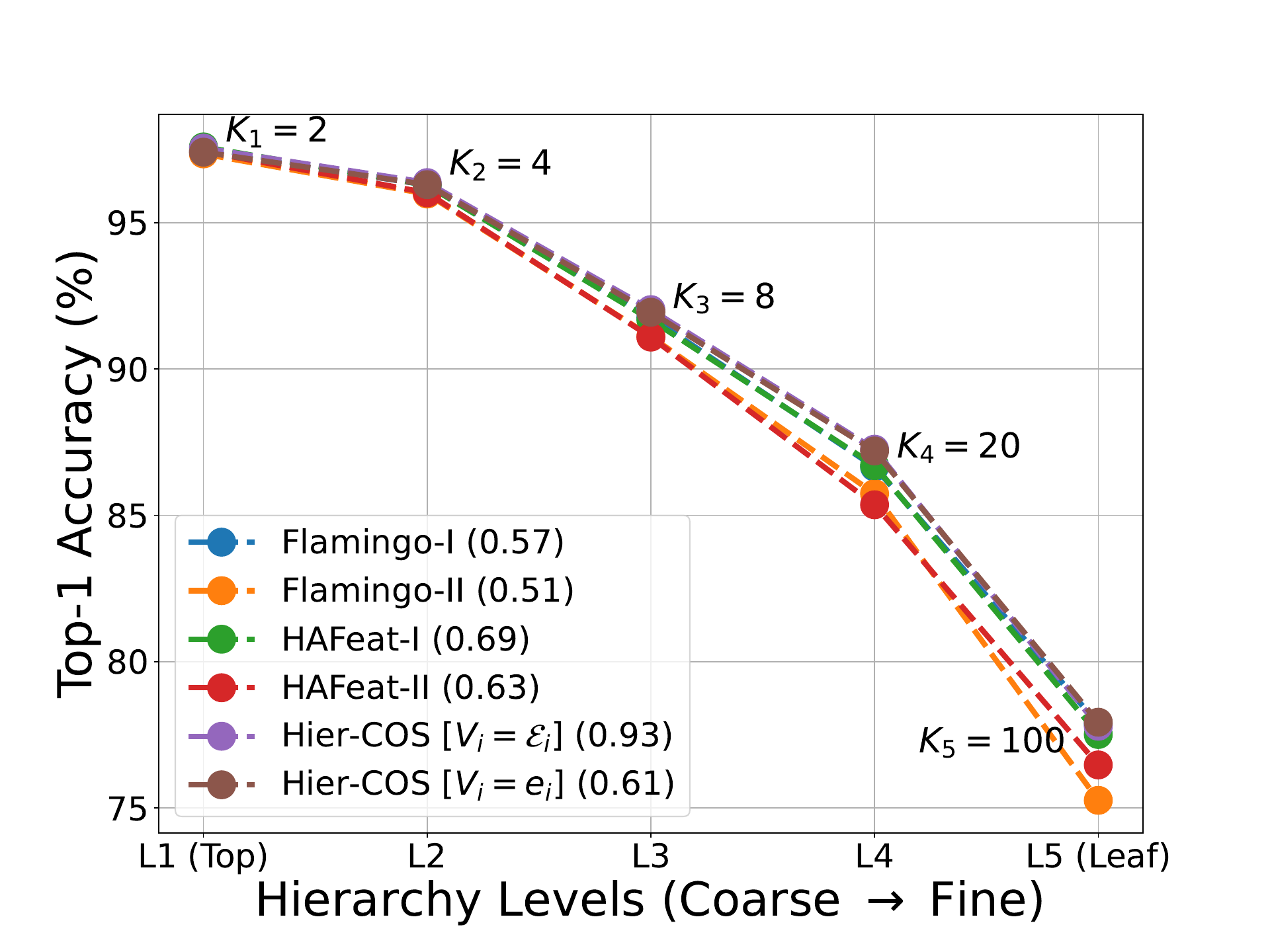}
        \caption{Level-wise accuracy for CIFAR-100}\label{fig:ablation-b}
    \end{subfigure}
    \begin{subfigure}[b]{0.33\textwidth}
        \centering
        \includegraphics[width=\linewidth]{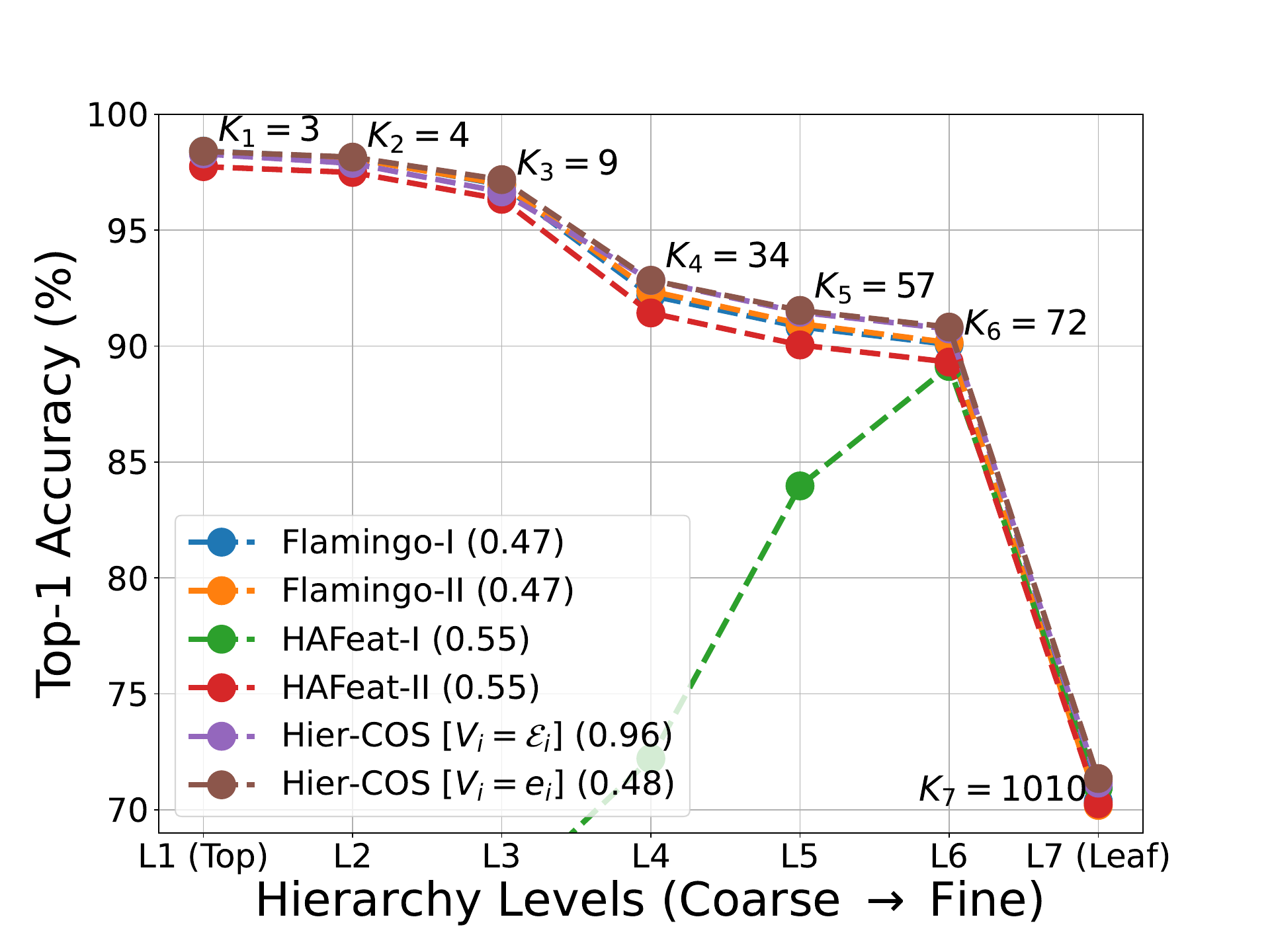}
        \caption{Level-wise accuracy for iNaturalist-19} \label{fig:ablation-c}
    \end{subfigure}
    \caption{Ablative analysis examining (a) the impact of the regularization term and (b, c) the effect of adding extra dimensions to $V_i$. HOPS metric is mentioned in parenthesis in (b, c).}
    \label{fig:ablation}
\end{figure*}

In this section, we examine the (i) impact of the regularization term (Fig. \ref{fig:ablation-a}), (ii) the effect of adding extra dimensions to the vector space (Fig. \ref{fig:ablation-b} and \ref{fig:ablation-c}), (iii) the influence of $\alpha$ on the hierarchical performance (Table \ref{tab:ablation}), and (iv) the influuence of $w_l$ on the classification performance (Table \ref{tab:ablation-wl}. 
In Fig. \ref{fig:ablation-a}, we plot the percentage of samples with correct order for the top-$k$ predictions against $k$. We observe that when $\mathcal{L}_{reg}$ is not used (dotted line), the hierarchical performance reduces with $k$, however, it is still better than the previous methods shown in {\color{blue} Fig. 1} of the main paper. Including $\mathcal{L}_{reg}$ minimizes the norm of the projection of a feature vector corresponding to a class $y_i$ onto its complementary subspace $V_i^\perp = \text{span}(\mathcal{E} \setminus \mathcal{E}_{y_i})$. Therefore, we would expect the feature vector to be closer to the subspace corresponding to the correct class $y_i$. We validate this on CIFAR-100 by calculating the cosine similarity between $\mathbf{x}$ and $\mathbb{P}_{\mathcal{E}_{y_i}} \mathbf{x}$. The cosine similarity obtained with and without $\mathcal{L}_{reg}$ is 0.97 and 0.87, respectively.

In Fig. \ref{fig:ablation-b} and \ref{fig:ablation-c}, we observe that when no additional dimensions are added to the vector space, i.e., $V_i = \text{span}(\{e_i\})$, the level-wise accuracy is competitive with the proposed Hier-COS framework (i.e., $V_i = \text{span}(\mathcal{E}_i)$); however, the HOPS metric is deteriorated significantly. This suggests that extra dimensions help in learning diverse and discriminative features across all the hierarchical levels.
Further, we emphasize that, unlike \cite{flamingo_2021_Chang, baseline-hafeat}, a single classifier is used to predict classes across all the levels. This demonstrates that learned feature representations are generalized across all the hierarchical levels. We discuss this property further in Section \ref{supp:discussion}.

We evaluate the performance of Hier-COS in Table \ref{tab:ablation} with different values of $\alpha$ on CIFAR-100. We observe that the performance is not too sensitive to $\alpha$ and maintains the same hierarchical performance across all the values with slight variations in top-1 accuracy.

\begin{table}[!htp]\centering
\resizebox{\linewidth}{!}{
    \begin{tabular}{l|c|c|c|c|c}\toprule
    \textbf{$\alpha$} &\textbf{Accuracy} &\textbf{MS} &\textbf{AHD@5} &\textbf{AHD@20}  &\textbf{HOPS}\\\midrule
    \textbf{0} &77.51 $\pm$ 0.288	&2.21 $\pm$ 0.026	&1.10 $\pm$ 0.005	&2.17 $\pm$ 0.002  &0.93 $\pm$ 0.001\\
    \textbf{5e-2} &77.79 $\pm$ 0.145 &2.21 $\pm$ 0.017 &1.09 $\pm$ 0.005 &2.17 $\pm$ 0.003  &0.93 $\pm$ 0.001
    \\
    \textbf{1e-2} &77.39 $\pm$ 0.231 &2.21 $\pm$ 0.049 &1.10 $\pm$ 0.010 &2.18 $\pm$ 0.004  &0.93 $\pm$ 0.002
    \\
    \textbf{5e-3} &77.43 $\pm$ 0.072 &2.21 $\pm$ 0.008 &1.10 $\pm$ 0.002 &2.18 $\pm$ 0.002  &0.93 $\pm$ 0.001
    \\
    \textbf{1e-3} &77.48 $\pm$ 0.371 &2.22 $\pm$ 0.020 &1.10 $\pm$ 0.006 &2.18 $\pm$ 0.004  &
    0.93 $\pm$ 0.002\\
    \textbf{1e-4} &77.41 $\pm$ 0.217 &2.23 $\pm$ 0.019 &1.11 $\pm$ 0.005 &2.18 $\pm$ 0.004  &0.93 $\pm$ 0.001\\
    \bottomrule
    \end{tabular}
}
\caption{Sensitivity Analysis of $\alpha$ on CIFAR-100}\label{tab:ablation}
\end{table}

Finally, in Table \ref{tab:ablation-wl}, we analyze the impact of distributing the weights $w_l$ such that (i) $w_l > w_{l+1}$, i.e., projection norm is concentrated towards the coarser classes, (ii) $w_l = w_{l+1}$, i.e., projection norm is concentrated equally across all the levels, and (iii) $w_l < w_{l+1}$, i.e., projection norm is concentrated towards the finer classes. To obtain weights $w_l > w_{l+1}$, we simply reverse the order of the level-wise weights obtained using the equation mentioned in the main paper, i.e., $w_l = \text{exp}\bigg(\frac{1}{h+1-l}\bigg)$. Obtaining the uniform weights for all the levels is straightforward and done using $w_l = \frac{1}{h}$. Empirical evidence supports our discussion in \textcolor{blue}{Sec. 3.3} of the main paper, that the leaf classes become indistinguishable as the concentration of the projection norm becomes high towards the coarser classes. This also justifies our design choice of a monotonically increasing weight function as we move from the basis vectors of the root node to the leaf. 

\begin{table}[!htp]\centering
\resizebox{\linewidth}{!}{
    \begin{tabular}{c|c|c|c|c|c}\toprule
    \textbf{Weights} &\textbf{Accuracy} &\textbf{MS} &\textbf{AHD@1} &\textbf{AHD@5} &\textbf{AHD@20} \\\midrule
    $w_l > w_{l+1}$ &51.08 $\pm$ 19.681 &1.92 $\pm$ 0.070 &0.93 $\pm$ 0.338 &1.29 $\pm$ 0.156 &2.24 $\pm$ 0.056 \\
    $w_l = w_{l+1}$ &70.22 $\pm$ 2.860 &2.08 $\pm$ 0.019 &0.62 $\pm$ 0.055 &1.16 $\pm$ 0.029 &2.19 $\pm$ 0.010 \\
    $w_l < w_{l+1}$ &77.79 $\pm$ 0.145 &2.21 $\pm$ 0.017 &0.49 $\pm$ 0.006 &1.09 $\pm$ 0.005 &2.17 $\pm$ 0.003 \\
    \bottomrule
    \end{tabular}
}
\caption{Impact of the distribution of weights $w_l$ on Top-1 accuracy and other hierarchical metrics}\label{tab:ablation-wl}
\end{table}

\subsection{Qualitative Analysis}
In this section, we demonstrate that our method learns stronger hierarchical feature representations as compared to the previous methods. {\color{blue} Fig. 1} of the main paper, only highlights the overall hierarchical performance of the methods, where even a small mistake would result in an error. For a deeper analysis of the learned representations, for all the samples belonging to a class $y_i$, we visualize the average of the log probabilities for the complete prediction vector comprising of $K$ classes. A higher log probability denotes a higher preference (or a low inter-class distance). We explain these plots using the ground truth LCA-based log probabilities on the test set of CIFAR-100 in Fig. \ref{fig:cifar-100-gt}, representing the best hierarchical performance, and demonstrate the hierarchical performance across all the datasets in Fig. \ref{fig:log-prob}. We observe that Hier-COS is able to learn better hierarchical representations than any of the other baseline methods.

\begin{figure}
    \centering
    \includegraphics[width=0.5\linewidth]{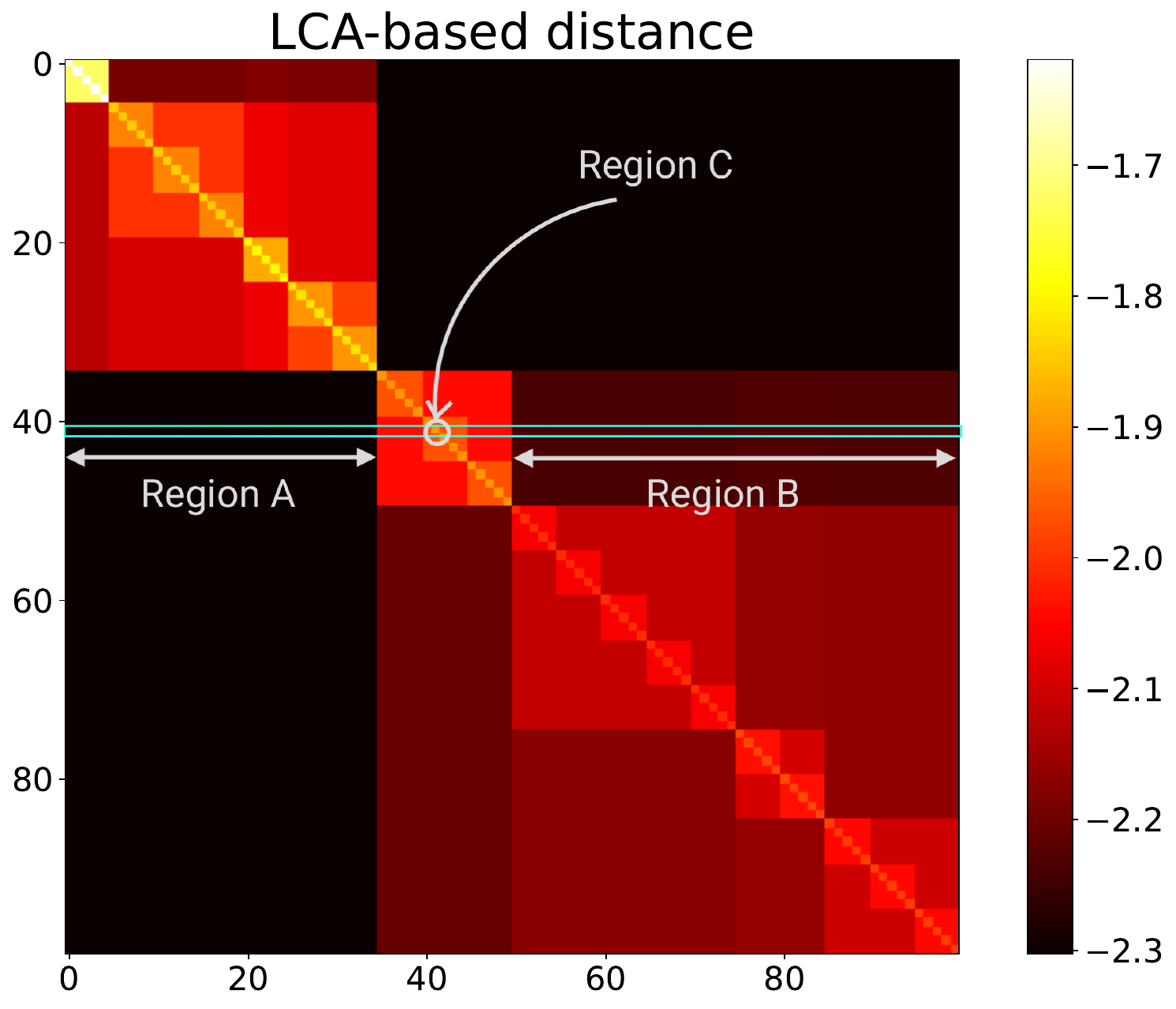}
    \caption{Ground truth LCA-based log probabilities on the test set of CIFAR-100 representing the best hierarchical performance.
    Consider the row highlighted by the rectangular box, say class $y_c$. `Region A' contains the classes farthest from the class $y_c$ because they have the least log probability in that row. Similarly, `Region B' contains all the classes closer to $y_c$ than the classes in `Region A' but are farther than all the other classes that are not in `Region A' and `Region B'. `Region C' contains only one class, i.e., $y_c$, and represents the average log probability of being correct. The color scale of log probabilities helps in understanding the hierarchical preference order. Therefore, for learned hierarchy-aware feature representation, we would expect the log probabilities to follow a similar preference order.
    }
    \label{fig:cifar-100-gt}
\end{figure}

\begin{figure}
    \centering
    \begin{subfigure}[b]{\linewidth}
        \includegraphics[width=\linewidth]{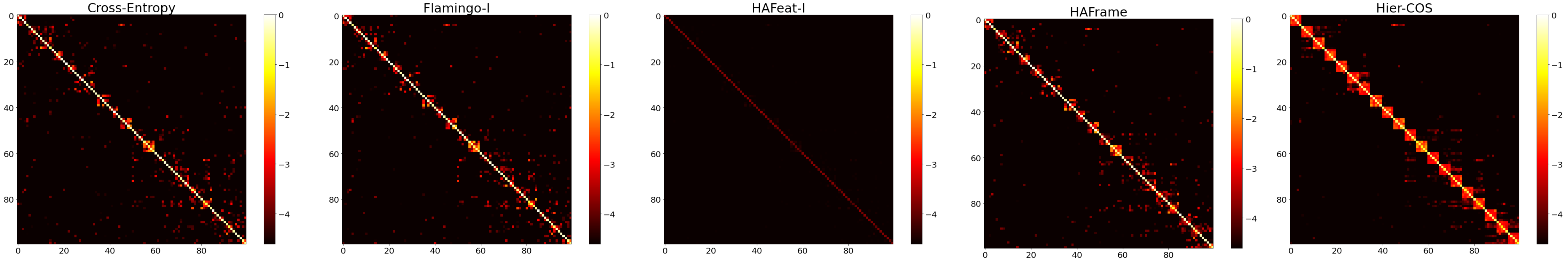}    
        \caption{CIFAR-100}
    \end{subfigure}
    \begin{subfigure}[b]{\linewidth}
        \includegraphics[width=\linewidth]{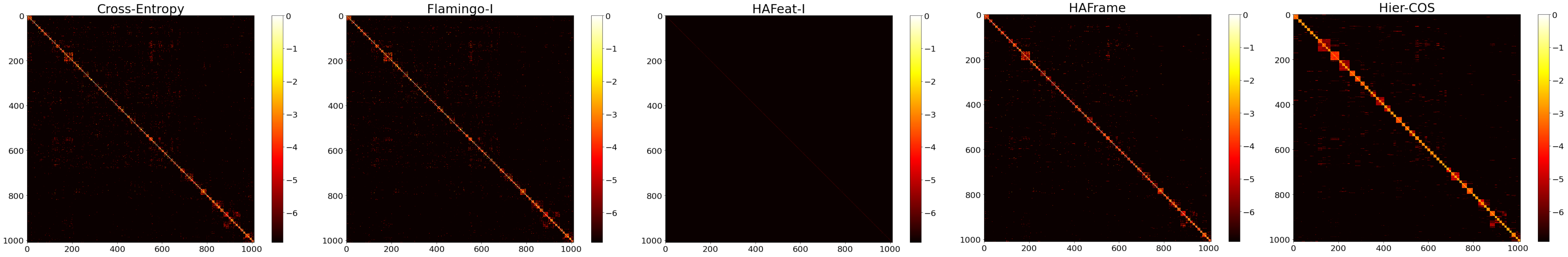}    
        \caption{iNaturalist-19}
    \end{subfigure}
    \begin{subfigure}[b]{\linewidth}
        \includegraphics[width=\linewidth]{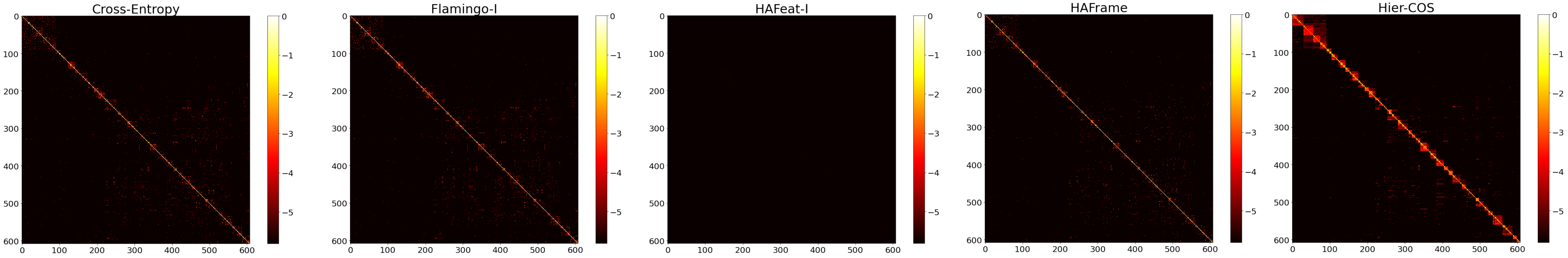}    
        \caption{tieredImageNet-H}
    \end{subfigure}
    \caption{\textbf{(Zoom in for better clarity)} Plots showing the obtained order from different classification methods on the test set of (a) CIFAR-100, (b) iNaturalist-19 and (c) tieredImageNet-H. The rows and columns represent the K classes. Each row contains the softmax-based log probabilities for all the K classes averaged over all the samples belonging to that class. All the classes are grouped according to the LCA distance.}
    \label{fig:log-prob}
\end{figure}

%% file: sec/X_supp_cvpr_lit_review.tex
\section{A Review of Evaluation Metrics} \label{supp:review}
This section discusses some of the widely used and related evaluation metrics for hierarchical classification methods. This also includes metrics used for evaluating information retrieval and recommendation systems.

\subsection{Definitions}
\noindent \textbf{Mistake Severity:}
The severity of a mistake is often quantified using the hierarchical distance of a mistake, i.e., the LCA distance between the ground truth and the predicted class when the input is misclassified \cite{bertinetto_2020_CVPR}.

\begin{align*}
    MS = \frac{1}{FP + FN} \sum_{i=1}^{\vert \mathcal{X} \vert} height(LCA(y_i, \widehat{y}_i ))
\end{align*}
where $y_i$ and $\widehat{y}_i$ are the ground truth and the predicted class label at the finest level of granularity for a sample $i$ in the dataset and $\vert \mathcal{X} \vert$ is the total number of samples in the dataset $\mathcal{X}$. The MS is averaged over all misclassifications, i.e., the total number of false positives (FP) and false negatives (FN). It is important to note that the hierarchical distance for correct predictions is zero, and hence, averaging it over only FP and FN makes it equivalent to saying that MS is computed for misclassified samples only. \\

\noindent \textbf{Average Hierarchical Distance @$k$:}
A major drawback of the MS metric was pointed out by \cite{karthik_2021_ICLR} that it only considers the misclassified samples and hence needs to be paired with top-1 accuracy for comparing different methods. A simple extension of MS is to average over all the samples $\vert \mathcal{X} \vert$ instead of only the misclassified ones (FP + FN) \cite{bertinetto_2020_CVPR}.

\begin{align*}
    AHD = \frac{1}{\vert \mathcal{X} \vert} \sum_{i=1}^{\vert \mathcal{X} \vert} height(LCA(y_i, \widehat{y}_i ))
\end{align*}

In the case of hierarchy-aware feature representations, we would expect not only lower MS and AHD values but also smaller LCA distances between the top predictions -- an aspect that neither MS nor AHD adequately captures. Therefore, \cite{bertinetto_2020_CVPR} also proposes to compute the average AHD for all the top-k predictions, defined as: 

\begin{align*}
    AHD@k = \frac{1}{\vert \mathcal{X} \vert} \sum_{i=1}^{\vert \mathcal{X} \vert} \frac{1}{k} \sum_{j=1}^{k} height(LCA(y_i, \widehat{y}_{ij} ))
\end{align*}
where $\widehat{y}_{ij}$ is the top-$j^{th}$ ranked prediction for sample $i$.

\noindent \textbf{Hierarchical Precision and Recall:}
For evaluating hierarchical classification methods, hierarchical precision ($hP$) is used, which is the proportion of correctly predicted classes among all the predicted classes at all the hierarchical levels, 
\begin{align*}
    hP = \frac{1}{\vert \mathcal{X} \vert}\sum_{i=1}^{\vert \mathcal{X} \vert}\frac{\vert \tilde{f_a}(v_{\widehat{y}_i}) \cap \tilde{f_a}(v_{y_i}) \vert }{\vert \tilde{f_a}(v_{\widehat{y}_i)} \vert}
\end{align*}
where $y_i$ and $\widehat{y}_i$ are the ground truth and predicted class labels, respectively, $v_{y_i}$ and $v_{\widehat{y}_i}$ are the corresponding nodes in the tree and $\tilde{f_a}(v_i)$ is the set containing all the ancestors of $v_i$ and itslef, i.e., $\{f_a(v_i) \cup v_i\}$ \cite{Silla_DMKD2011}. 

Similarly, for any sample $i$, hierarchical recall ($hR$) is the proportion of correctly predicted classes among all the ground truth classes at all the hierarchical levels.

\begin{align*}
    hR = \frac{1}{\vert \mathcal{X} \vert}\sum_{i=1}^{\vert \mathcal{X} \vert}\frac{\vert \tilde{f_a}(v_{\widehat{y}_i}) \cap \tilde{f_a}(v_{y_i}) \vert }{\vert \tilde{f_a}(v_{y_i)} \vert}
\end{align*}
In simpler words, we can say that $hR$ is analogous to $hP$ but focuses on how well the true hierarchical structure is retrieved rather than the correctness of the predicted structure. \\

\noindent \textbf{Total Preference Ordering:}
We often deal with information expressed as partial orderings (PO). For PO based on a hierarchy tree, the highest preference is given to the ground truth class, the second highest to all its siblings, the third highest to all its first cousins, and so on, resulting in a total preference ordering (TPO) \cite{dezert2024distancespartialpreferenceorderings}. There could also be other TPOs, such as TPO based on the prediction probabilities.
\cite{dezert2024distancespartialpreferenceorderings} presents a metric to compare TPOs using a pair-wise Preference-Score Matrix (PSM), $M$, for each TPO, defined as:

\begin{align*}
    M(i, j) &= 
    \begin{cases} 
    1, & \text{if } x_i \succ x_j, \\ 
    -1, & \text{if } x_i \prec x_j, \\ 
    0, & \text{if } x_i = x_j.
    \end{cases}
\end{align*}

where $x_i$ and $x_j$ denote the preference scores of a classes $i$ and $j$, respectively, and $M(i,j)$ is the component of PSM representing whether the classes $i$ and $j$ are in the correct preference order.
Finally, \cite{dezert2024distancespartialpreferenceorderings} defines the distance between the two TPOs using the Frobenius distance as follows:
\begin{align*}
    d_F(M_1, M_2) &= \lVert M_1 - M_2\rVert _F \\
    &= \sqrt{\operatorname{Tr}\bigg((M_1 - M_2)^T (M_1 - M_2)\bigg)}
\end{align*}

\noindent \textbf{Mean Reciprocal Ratio and Normalized Rank:}
A popular evaluation metric used in information retrieval and ranking systems is Mean Reciprocal Rank (MRR). We present its definition with respect to our problem setup. As the name suggests, it measures the mean of the reciprocal of the ranks for the correct class. Specifically, rank can be considered as the position or index at which the correct class appears in an ordered list of predictions. The lowest rank, zero, signifies the prediction with the highest confidence. 

\begin{align*}
    MRR = \frac{1}{\vert \mathcal{X} \vert} \sum_{i=1}^{\vert \mathcal{X} \vert} \frac{1}{rank_i}
\end{align*}

Recently, Tian et al. \cite{mnr-ndcg} proposed an alternative to MRR for hierarchical classification called Mean Normalized Rank (MNR). We present its definition with respect to multi-class classification below:
\begin{align*}
    MNR = \frac{1}{\vert \mathcal{X} \vert} \sum_{i=1}^{\vert \mathcal{X} \vert} \bigg( \frac{1}{H} \sum_{l=1}^H \bigg( \frac{rank_i^l}{K_l} \bigg) \bigg)
\end{align*}
In simpler words, MNR is the average rank over all the hierarchical levels normalized by the number of classes at each level.

\noindent \textbf{Normalised Discounted Cumulative Gain:}
Normalized Discounted Cumulative Gain (NDCG) is a standard metric that evaluates the quality of recommendation and information retrieval systems. NDCG measures the ability of a method to sort items based on relevance. The relevance score is inversely proportional to ranks, i.e., the higher relevance score is better. The NDGC@$k$ computes the NDCG for top-$k$ predictions, given for each sample $i$ by:

\begin{align*}
    NDCG_i@k &= \frac{DCG_i@k}{IDCG_i@k} \\
    DCG_i@k &= \sum_{j=1}^k \frac{rel_{i,j}}{log_2(j + 1)}
\end{align*}
where, $IDCG@k$ is the DCG@$k$ for the ideal ranking. Recently, Tian et al. \cite{mnr-ndcg} presented an approach to extend this metric to hierarchical classification, where relevance between classes $i$ and $j$ is defined based on the structure of the hierarchy tree as:

\begin{align*}
    rel_{i,j} = 1 - \frac{d(i, lca(i,j)) + d(j, lca(i,j))}{D_\mathcal{T}}
\end{align*}
where, $i$ is the ground truth class, $j$ is the predicted class, $lca(i,j)$ is the lowest common ancestor between $i$ and $j$, $d(i,j)$ is the number of edges between the nodes $i$ and $j$, and $D_\mathcal{T}$ is the length of the longest path between any two leaf nodes.

\begin{figure*}[!t]
    \centering
    \includegraphics[width=\linewidth]{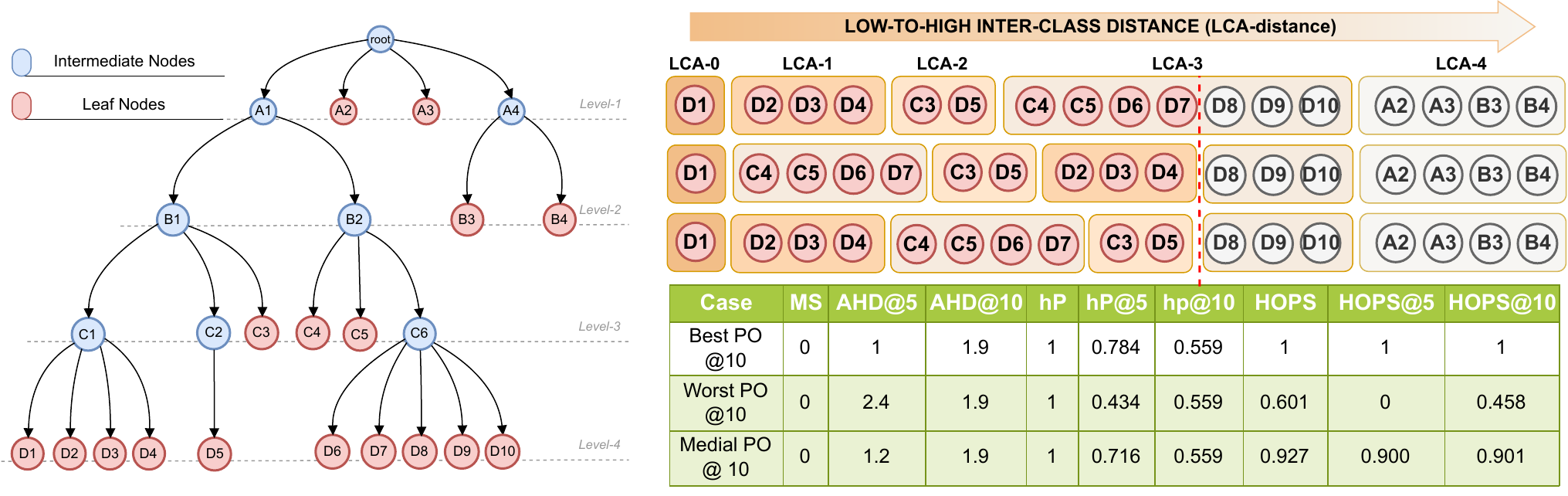}
    \caption{(Left) The same taxonomy $\mathcal{T}$ shown in \textcolor{blue}{Fig. 2} of the main paper. (Right) Three examples with correct class prediction but varying the predicted preference order for top-10 predictions.}
    \label{fig:hops-analysis}
\end{figure*}

\subsection{Limitations}
\noindent \textbf{Mistake Severity:}
We emphasize that MS only considers the average cost of mistakes and is biased towards methods that make \textit{more} mistakes in numbers but less in severity \cite{karthik_2021_ICLR}. Therefore, it can only be compared when coupled with the top-1 accuracy. This is also evident for the `Soft-labels $\beta=4$' and `Barz \& Denzler*' methods in {\color{blue}Tables 2 and 3} of the main paper, respectively. Moreover, the metric is dependent on the properties of the tree and is not normalized. Because the MS is unnormalized, this is often interpreted as the average LCA distance for a mistake. However, for a given class, the lowest MS for misclassification is dependent on the LCA distance of the nearest sibling, which results in a higher MS for imbalanced trees, even for less severe mistakes. For instance, whenever classes $A2$ and $A3$ of the taxonomy $\mathcal{T}$ in Fig \ref{fig:hops-analysis} are misclassified, the MS is expected to be high because most incorrect classes are at a higher LCA distance. Therefore, we discover that MS depends on the leaf node's depth, tree's branching factor and imbalance, and it is challenging to interpret the MS score without due consideration to the tree structure.

\noindent \textbf{AHD@$k$:}
AHD@1 can be viewed as a straightforward extension of MS that does not require pairing with top-1 accuracy. However, as highlighted by \cite{baseline-hafeat}, this metric tends to be biased towards higher top-1 accuracy and fails to effectively measure the severity of mistakes.
Moreover, similar to MS, AHD@$k$ is also unnormalized and suffers from the limitations discussed above. For instance, although the predictions are correct, AHD@\{5, 10\} is non-zero for all the cases in the Fig. \ref{fig:hops-analysis}, which depends on the structure of the tree. Moreover, the average operation is permutation invariant; therefore, any random permutation of the same top-k prediction results in the same AHD@$k$. For instance, AHD@10 for any random permutation of the top-10 predictions is always $1.9$, as shown in Fig. \ref{fig:hops-analysis}. Therefore, we discover that AHD@$k$ also depends on the properties of the tree and is invariant to the order of predictions.

\noindent \textbf{hP and hR:}
Similar to AHD@$k$, we can also use the average hP and hR of the top-k predictions to get hP@$k$ and hR@$k$, respectively. Both hP@$k$ and hR@$k$ can be viewed as alternatives to AHD@$k$ that are normalized. 
Observe that $\vert \tilde{f_a}(v_{\widehat{y}_i}) \cap \tilde{f_a}(v_{y_i}) \vert$ measures the LCA similarity between the ground truth and the predicted classes, i.e., as the number of common nodes increases the similarity increases; therefore, a higher value represents better performance. However, because of their resemblance to AHD@$k$, both possess similar limitations, i.e., both depend on the tree's properties and are permutation invariant. Moreover, for a balanced tree having the leaf nodes at the same level, we know that $\vert \tilde{f_a}(v_{\widehat{y}_i}) \vert = \vert \tilde{f_a}(v_{y_i}) \vert$, therefore, $hP = hR$. In Fig \ref{fig:hops-analysis}, it is evident that hP@$k$ is a normalized version of AHD@$k$.

\noindent \textbf{TPO:}
This metric fundamentally depends on the definition of TPOs and their corresponding PSMs. We can use the hierarchy tree to define a ground truth PSM, say $M_1$. However, determining an order of preferences from the prediction is not well-defined when we have multiple classes at the same preference value, which is almost always the case with hierarchies. Therefore, although \cite{dezert2024distancespartialpreferenceorderings} presents a promising direction to compare the TPOs, it is not yet feasible for hierarchical evaluations. Even if we assume a method exists to determine the order of preferences from the prediction, using this method for hierarchical evaluation is still challenging. The construction of PSM considers all misclassifications equally severe. Moreover, deciding which classes (based on ground truth TPO) are more important when considering a hierarchy with a huge number of leaf nodes is often important.

\noindent \textbf{MRR and MNR:}
MRR only considers the rank of the correct prediction, while ignoring all the negative classes. Therefore, it is analogous to AHD@1 and, hence, suffers from the same problems. Similarly, MNR is also computed over the correct prediction at all levels and has the same limitations. Further, MNR requires a ranking over all the hierarchical levels. Hence, it can not be used for methods that predict for a single level, like Cross-Entropy and HAFrame.

\noindent \textbf{NDCG@$k$: }
The definition of the relevance given by \cite{mnr-ndcg} does not hold for trees that do not have leaf nodes at the same level. According to them, $d(D1, A2) < d(D1, D6)$ in Fig \ref{fig:hops-analysis}, however, we can see that $D1$ is closer to $D6$ than $A2$. Moreover, the discounting factor, $\frac{1}{log_2(j+1)}$ for each class with rank $j$, is independent of the hierarchical structure. Therefore, NDCG@$k$ does not capture the properties of the hierarchy tree. Specifically, for a class that has all the nearest-$k$ neighbors at the same LCA distance, if the predicted class at $(k-1)^\text{th}$ rank has a higher LCA distance, the penalty is decayed significantly $(\frac{1}{log_2(k)})$ without considering that all the top-$k$ predictions are equally relevant.

%% file: sec/X_supp_cvpr_hops.tex
\section{HOPS: Technical Description \& Example} \label{supp:hops}

Let $\mathcal{U}_{d, y_c}$ be the set of classes at the finest level that are at LCA distance $d$ from $y_c$. Note that $\mathcal{U}_{0, y_c}$ is a singleton set with the only element being $\{y_c\}$, i.e., the correct class. We emphasize that the cost of misclassifying $y_c$ as any of the classes in $\mathcal{U}_{d,y_c}$ is the same, i.e., $d$. Therefore, an order on $\mathcal{U}_{d,y_c}$ can be used to define the preferences of classes where any order within a set $\mathcal{U}_{d,y_c}$ is equally preferred, i.e., the preferred ordering will be $\{\mathcal{U}_{0,y_c}, \cdots, \mathcal{U}_{H,y_c}\}$. In the case when a node only has one child, for some value $d=l$, $\mathcal{U}_{l,y_c}$ will be empty and will be of no significance. 
In order to ignore such empty sets, we use ranks instead of absolute LCA distances to represent the ordering. Specifically, we use an index set $\bbI$ to rank all the non-empty sets in $\{\mathcal{U}_{d, y_c}\}_{d=0}^{d=H}$. We define the index set $\mathbb{I} = \{0,1,2,\ldots,H\}$ for indexing the elements in the ordered set $\mathbb{S}_{y_c}$, defined as
\begin{equation}
    \mathbb{S}_{y_c} = \{\mathbb{S}_{y_c}^i \mid i \in \mathbb{I}\} := \{\mathcal{U}_{d, y_c} \mid \mathcal{U}_{d,y_c} \neq \phi\}_{d=0}^{d=H}.
\end{equation}
Now, we define the \emph{desired order} $z$ for a class $y_c$ as:
\begin{equation}
    z = [ 0, \underbrace{1, \dots, 1}_{|\mathbb{S}_{y_c}^1|\text{ times}}, \cdots, \underbrace{k, \dots, k}_{|\mathbb{S}_{y_c}^k|\text{ times}} ]
    \label{eqn:pref_order}
\end{equation}
Note that this preference order is conditioned on the true class $y_c$. $\vert \mathbb{S}_{y_c}^i \vert$ denotes the number of classes at rank $i$; therefore, $z$ encodes the properties of the tree via ranks. 

Let $\Pi = [ \pi_{y_1}, \cdots, \pi_{y_K} ]$ be the vector of probabilities corresponding to all the leaf classes and $\hat{\Pi} = \arg \text{sort}(\Pi)$ be the vector of indices (class labels) that would sort $\Pi$ in descending order. Assume $rank(y_c, y_j)$ returns the index $i \in \mathbb{I}$ of the set $\mathcal{U}_{d,y_c}$ that contains the class $y_j$.
Finally, we define the \emph{predicted order} $\widehat{z}$ as follows:
\begin{equation}
    \widehat{z} = [rank(y_c, \hat{\Pi}_j) ]_{j=1}^{K}
\end{equation}
By defining a metric on the \textit{desired} and \textit{predicted} preference orders, $z$ and $\widehat{z}$ respectively, we can derive a hierarchical metric that takes into account the branching factor and depth (or height) of the tree. Specifically, we compute the following for a single sample having the true class $y_c$ as:
\begin{align} \label{eq:hops}
    s = &\sum_{j=1}^{K} \eta_j \cdot |z_j - \widehat{z_j}|
\end{align}
where, $\eta_j$'s are defined using a multi-step exponential-linear decay function. Specifically, we use the exponential decay when the preferred rank changes, i.e., $\eta_j=2^{-z_j}$ when $j$ is the first occurrence of rank $z_j$, and for each subsequent values $j$ for rank=$z_j$, we use a linear decay weight by linearly interpolating between $2^{-{z_j}}$ and $2^{-(z_j+1)}$. An exponential decay ensures that the classes with higher LCA distances, when ranked closer to the predicted class, have a relatively high negative impact on the score; thereby accounting for the imbalance of the tree. Meanwhile, the linear decay accounts for the branching factor by reducing the weights for classes at the same rank as we move away from the correct class. The rationale behind the choice of this decay function is detailed in the supplementary.

Note that we get $s_{min}=0$ when $z=\widehat{z}$, and we use the worst possible predicted order, i.e., the reverse order of $z$, to get the maximum achievable score $s_{max}$. The final measure of HOPS is given by:
\begin{align}
    \text{HOPS}=1 - \frac{s}{s_{max}} ~~~\in [0,1]
\end{align}

\textbf{HOPS@$k$}. 
We are often interested in only the top-$k$ performance of a classifier, particularly, when there are a large number of classes. We define the HOPS@$k$ variants, simply by modifying the notion of the worst achievable score and adapting the weights. For HOPS@$k$, we define the worst ranked vector $z^*_r = [z_{k:-1:1},z_{k+1:1:K}]$ \footnote{In $z_{a:c:b}$, $a$ is the start, $b$ is the stop and $c$ is the step-size} yielding the score as $s^k_{max}$ which is then used to compute HOPS@$k$: 
\begin{equation}
    \text{HOPS}@k = \max\left(0,1-\frac{s^k}{s^k_{max}}\right)
\end{equation}
where $s^k$ is computed using (\ref{eq:hops}), but with $\eta_{j>k}=0$. 
The scores averaged over all samples are used for reporting the classifiers performance. It is worth pointing out that for $k=1$, the HOPS@$k$ is equal to the top-1 accuracy. This is aligned with the design of this metric, where the HOPS@$k$ concerns with the ranking performance of up to the $k^{th}$ ranked prediction, while the HOPS metric concerns the complete ranking over all classes.

\subsection{Example}
As discussed in the previous section, for a given ground-truth class $y_c$, we get a \emph{preference order} based on the hierarchy tree by constructing $z$ as:
\begin{equation}
    z = [0, \underbrace{1, \dots, 1}_{\times n_1}, \cdots, \underbrace{k, \dots, k}_{\times n_{k}}]\quad ,~k\le H
\end{equation}
where, $n_k$ denotes the number of $(k-1)^{th}$-\emph{nearest} higher order cousins of class $y_c$ and $n_0 = 1$ (self). This implies that $z$ is the \textit{desired} preference ordering where all the classes are arranged according to the hierarchy-based preferences such that a lower value of $z_j$ denotes a higher preference (similar to ranks). We then propose to construct $\widehat{z}$ using the predictions such that it contains the preference values of the predicted classes based on $z$. For instance, $z$ and $\widehat{z}$ corresponding to the second row in Fig \ref{fig:hops-analysis} will be:

\begin{align*}
    z &= [0, 1, 1, 1, 2, 2, 3, 3, 3, 3, 3, 3, 3, 4, 4, 4, 4] \\
    \widehat{z} &= [0, 3, 3, 3, 3, 2, 2, 1, 1, 1, 3, 3, 3, 4, 4, 4, 4]
\end{align*}